\documentclass[12points]{article}

\usepackage[utf8]{inputenc}
\usepackage[T1]{fontenc}
\usepackage[hyphens]{url}
\usepackage{authblk}
\makeatletter
\renewcommand\AB@affilsepx{\quad\protect\Affilfont}
\makeatother
\renewcommand\Affilfont{\fontsize{10}{13}\selectfont}
\usepackage{booktabs}
\usepackage{amsfonts}
\usepackage{pifont}
\usepackage{nicefrac}
\usepackage{microtype}
\usepackage{amsmath, amsthm, amssymb}
\usepackage[
    colorlinks=true,
    linkcolor={red},
    citecolor={blue!50!black},
    urlcolor={blue!70!black}
]{hyperref}
\usepackage{tcolorbox}
\usepackage{fancybox}
\usepackage{graphicx}
\usepackage{listings} 
\usepackage{float}
\usepackage{algorithm}
\usepackage[noend]{algpseudocode}
\usepackage{color}
\usepackage{tikz}
\usepackage{pgfplots}
\usepackage{array}
\usepackage{caption}
\usepackage{subcaption}
\usepackage[textsize=scriptsize]{todonotes}
\usepackage{enumitem}
\usepackage{wrapfig}
\usepackage{etoolbox}
\usepackage{diagbox}
\usepackage{comment}
\usepackage{makecell}
\usepackage{multirow}
\usepackage{bbm, bm}
\usepackage{multicol}
\usepackage{mathpazo}
\usepackage{array, booktabs}
\usepackage{arydshln}
\usepackage{subcaption}
\usepackage{siunitx}
\usepackage{times}
\usepackage{setspace}
\usepackage[normalem]{ulem}
\usepackage[round]{natbib}
\usepackage{fontawesome5}
\usepackage{cleveref}
\usepackage{aliascnt}
\usepackage{xcolor}
\newaliascnt{appendixsection}{section}
\crefname{appendixsection}{appendix}{appendices}
\Crefname{appendixsection}{Appendix}{Appendices}
\usepackage[
  textwidth=6.0in,
  textheight=9.0in,
  left=1.25in,
  top=1in
]{geometry}

\setlist[itemize]{leftmargin=2.5em}
\setlist[itemize]{leftmargin=2.5em}
\setlist[enumerate]{leftmargin=2.5em}
\usepackage[framemethod=TikZ]{mdframed}
\mdfsetup{skipabove=5pt,
innertopmargin=0pt}

\tcbuselibrary{breakable}
\lstdefinestyle{modeloutput}{
  basicstyle=\ttfamily\tiny,
  breaklines=true,
  breakatwhitespace=false,
  breakindent=0pt,
  columns=fullflexible,
  keepspaces=true,
  literate={≈}{{$\approx$}}1
}

\newtcolorbox{responsebox}[1]{
  breakable,
  colback=gray!3,
  colframe=gray!40,
  title={#1},
  fonttitle=\bfseries,
  boxrule=0.4pt,
  arc=1mm,
  left=0.5mm,
  right=0.5mm,
  top=0.5mm,
  bottom=0.5mm,
  before skip=0.3em,
  after skip=0.3em
}

\newtheoremstyle{slplain}
  {.4\baselineskip\@plus.1\baselineskip\@minus.1\baselineskip}
  {.3\baselineskip\@plus.1\baselineskip\@minus.1\baselineskip}
  {\itshape}
  {}
  {\bfseries}
  {.}
  { }
  {}
\theoremstyle{slplain} 

\newtheorem*{definition*}{Definition}
\newtheorem*{theorem*}{Theorem}
\newtheorem{theorem}{Theorem}[section]
\newaliascnt{lemma}{theorem}
\newtheorem{lemma}[lemma]{Lemma}
\aliascntresetthe{lemma}

\newaliascnt{proposition}{theorem}

\aliascntresetthe{proposition}

\newaliascnt{corollary}{theorem}

\aliascntresetthe{corollary}

\newaliascnt{observation}{theorem}

\aliascntresetthe{observation}

\newaliascnt{claim}{theorem}

\aliascntresetthe{claim}

\newaliascnt{definition}{theorem}

\aliascntresetthe{definition}

\newaliascnt{fact}{theorem}

\aliascntresetthe{fact}

\crefname{lemma}{lemma}{lemmas}
\Crefname{lemma}{Lemma}{Lemmas}
\crefname{proposition}{proposition}{propositions}
\Crefname{proposition}{Proposition}{Propositions}
\crefname{corollary}{corollary}{corollaries}
\Crefname{corollary}{Corollary}{Corollaries}

\crefalias{lemma}{lemma}
\crefalias{proposition}{proposition}
\crefalias{corollary}{corollary}

\makeatletter
\newtheorem*{rep@theorem}{\rep@title}
\newcommand{\newreptheorem}[2]{%
\newenvironment{rep#1}[1]{%
 \def\rep@title{#2 \ref{##1}}%
 \begin{rep@theorem}}%
 {\end{rep@theorem}}}
\makeatother

\newreptheorem{theorem}{Theorem}
\newreptheorem{lemma}{Lemma}
\newreptheorem{claim}{Claim}
\newreptheorem{corollary}{Corollary}
\newreptheorem{proposition}{Proposition}

\theoremstyle{definition}

\newtheorem{remark}[theorem]{Remark}

\theoremstyle{plain} 

\numberwithin{equation}{section}

\newcommand{\R}{\mathbb{R}}

\DeclareMathOperator{\E}{\mathbb{E}}
\DeclareMathOperator{\1}{\mathbb{1}}
\DeclareMathOperator{\KL}{\mathrm{KL}}

\renewcommand\bar\overline

\DeclareMathOperator*{\argmax}{arg\,max}

\newcolumntype{C}[1]{>{\centering\let\newline\\\arraybackslash\hspace{0pt}}m{#1}}

\newcommand{\calD}{\ensuremath{\mathcal{D}}}

\newcommand{\calO}{\ensuremath{\mathcal{O}}}

\def\nd/{\textsuperscript{nd}}
\def\rd/{\textsuperscript{rd}}
\def\th/{\textsuperscript{th}}

\makeatletter
\def\nnil{\nil}
\newcounter{prob}

\newcounter{dual}

\makeatother

\newenvironment{prob*}{%
	\csname equation*\endcsname%
	\aligned%
}{%
	\endaligned%
	\csname endequation*\endcsname%
}

\title{InfoSFT: Learn More and Forget Less with Information-Aware Token Weighting}
\author[1]{Mahdi Sabbaghi\thanks{Correspondence can be made to: smahdi@seas.upenn.edu}}
\author[1]{George Pappas}
\author[2]{Adel Javanmard}
\author[1]{Hamed Hassani}

\affil[1]{University of Pennsylvania}
\affil[2]{University of Southern California}
\date{}

\definecolor{sftcolor}{RGB}{0,0,0}
\definecolor{dftcolor}{RGB}{160,70,30}
\definecolor{infocolor}{RGB}{0,81,0}

\begin{document}
\maketitle

\begin{abstract}
Supervised fine-tuning (SFT) provides the standard approach for teaching LLMs new behaviors from offline expert demonstrations. However, standard SFT uniformly fits all samples—including those with low likelihood under the base model—which can disproportionately drive training updates toward overfitting specific samples rather than learning the target behavior. Moreover, adapting to these unlikely samples induces substantial policy shifts that degrade prior capabilities. Existing methods mitigate this by filtering, regenerating, or down-weighting low-likelihood data. In doing so, they often suppress precisely the novel behaviors the base model has yet to learn.

We propose InfoSFT, a principled weighting scheme for the SFT objective that concentrates learning signals on maximally informative, medium-confidence tokens—those neither overly familiar to the base model nor too unlikely to cause instability. Requiring only a one-line modification to the standard token-wise loss, InfoSFT demonstrably improves generalization over vanilla SFT and likelihood-weighted baselines across math, code, and chain-of-thought tasks with diverse model families, while better preserving pre-existing capabilities.
\begin{center}
    \faGithub~\url{https://github.com/Helloworld10011/InfoSFT}
\end{center}
\end{abstract}

\vspace{-6pt}
\section{Introduction}
Large Language Models (LLMs) derive much of their practical effectiveness from multi-stage post-training. Post-training typically combines supervised fine-tuning (SFT) with reinforcement learning (RL). In the first stage, SFT trains the model on offline expert demonstrations, thereby teaching the model new behaviors and broadening its capabilities. 
In the second stage, RL uses reward signals to align the model's outputs by shifting its probability mass  toward higher-reward responses \citep{ouyang2022training,touvron2023llama}. These two stages play complementary roles. While RL refines the model toward preferred responses, the role of SFT remains critical for introducing the new behaviors \citep{yue2025does,casper2023open,yoshihara2025practical,javanmard2026theoretical}.

Despite its importance, 
recent work highlights two primary shortcomings of standard SFT: \textbf{(i)} it often suffers from weak test-time generalization. SFT attempts to imitate all offline demonstrations, including trajectories that contain low-likelihood tokens under the model. These tokens incur large negative-log-likelihood losses and can force the model to overfit specific samples rather than learning the new behavior \citep{wu2025generalization,ren2026rethinking}. In contrast, on-policy approaches like RL, that train on high-reward samples from the model’s own distribution, often generalize better \citep{tajwar2024preference,chu2025sft,chen2025retaining}. \textbf{(ii)} SFT is prone to catastrophic forgetting: adapting the model to new behaviors using low-likelihood offline data causes a large shift in model's output distribution, whereas methods that utilize high-likelihood or on-policy samples stay closer to the base model. This larger shift is associated with degradation of the model's prior capabilities \citep{chen2025retaining,shenfeld2025rl,lai2025reinforcement}. To mitigate these issues, prior work explores reweighting or filtering low-likelihood trajectories and tokens, as well as regenerating them via the model itself \citep{wu2025generalization,li2025beyond,chen2025retaining,shenfeld2026self}. 

However, 
down-weighting or filtering low-likelihood samples entails a tradeoff:
such samples are often the examples needed to teach the model new behaviors, whereas higher-likelihood and on-policy data mainly refine behaviors the model already exhibits \citep{yuan2025mitigating,yue2025does}. This raises a central question: 

\vspace{-2pt}
\begin{quote}
    \emph{How can we leverage offline expert data to learn new behaviors while avoiding overfitting or catastrophic forgetting?}
\end{quote}  
\vspace{-2pt}

We approach this question by studying how to adaptively weight offline samples according to their model likelihood. Standard SFT assigns uniform weight to all samples, while recent methods like DFT \citep{wu2025generalization} assign weights proportional to the model likelihood: \(w \propto q\). While this mitigates overfitting, it simply overlooks the low-likelihood samples that are needed for learning the new behavior.

In this work, we propose a principled weighting scheme for learning from expert data. We formulate supervised fine-tuning through a proximal update framework \citep{schulman2015trust,tomar2020mirror}, asking which tokens are most informative under a budget of distributional shift. We derive a weighting rule that prioritizes fitting the underlying data distribution rather than memorizing individual samples. The resulting weight takes the form \(
    w \propto q \Big[ C - \log \frac{q}{1-q} \Big]_+
\) where \(q\) is model likelihood of the demonstrated token and \(C\) is a proper constant. 
Compared to DFT~\citep{wu2025generalization}, this rule increases the relative importance of low-likelihood expert samples by a factor of $-\log q$. Thus, the optimal update accounts not only for whether an expert sample is likely under the model, but also for how much \emph{information} it provides ($-\log q$ is the surprisal of a sample). Motivated by this perspective, we propose \textbf{InfoSFT}, a simple novel variant of SFT that emphasizes informative low-likelihood tokens,gives little weight to already-solved tokens, and assigns vanishing weight to extremely unlikely predictions as \(q \to 0\).

\begin{figure*}
    \centering
    \includegraphics[width=1.0\linewidth]{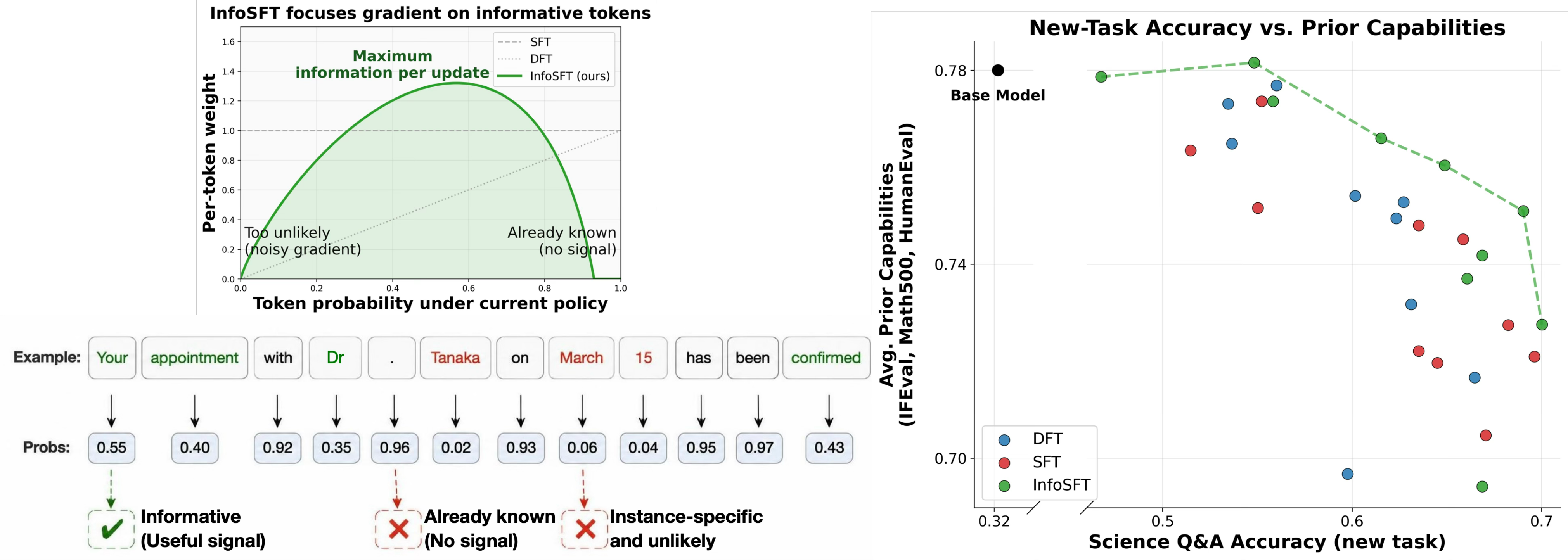}
    \vspace{-8pt}
    \caption{\textbf{(left)} InfoSFT assigns the highest weights to middle-confidence tokens as opposed to SFT that places the weights uniformly, and DFT that favors high-likelihood samples. \textbf{(right)} Shows the results for training on Science Q\&A \citep{shenfeld2026self}. By sweeping over several hyper-parameters like learning-rates and epochs, we show that checkpoints of InfoSFT achieve the best curve on the new-task/prior-capability tradeoff. 
    }
    \label{fig:intro}
    \vspace{-6pt}
\end{figure*}

\noindent\textbf{Our contributions} are summarized as follows: 

\begin{itemize}[leftmargin=0.8em]
    \vspace{-4pt}
    \item \textbf{Best update under a fixedh budget.}
    We study the tradeoff between learning new tasks and retaining  prior capabilities. In \Cref{sec:theory}, we show that, when the update is constrained to remain close to the base model, it must assign the highest weight to samples with intermediate confidence for the best performance gain. This is illustrated in \Cref{fig:intro} (left).
    
    \item \textbf{InfoSFT algorithm.}
    We propose InfoSFT, a simple supervised fine-tuning method that incorporates our weighting rule with only a one-line change to the standard token-wise loss.
    
    \item \textbf{Experimental results.}
    We evaluate InfoSFT across math, code, and CoT fine-tuning settings using several model families and sizes. InfoSFT consistently improves generalization over SFT and DFT in standard math and code fine-tuning, and can be combined with SFT to improve reasoning and get better results on AIME. For catastrophic forgetting, we show that InfoSFT achieves a better new-task/prior-capability tradeoff, as illustrated in \Cref{fig:intro} (right).
\end{itemize}

\vspace{-4pt}
\subsection{Related Work}
\textbf{Supervised fine-tuning.}
Supervised fine-tuning (SFT) is the default method to adapt pretrained LLMs to instructions and new tasks before alignment \citep{ouyang2022training,touvron2023llama,chung2024scaling}. The maximum log-likelihood objective is simple and enables learning from offline curated datasets such as Self-Instruct, Alpaca, and LIMA \citep{wang2023self,alpaca2023,zhou2023lima}. However, SFT also has several weaknesses: it can copy the teacher style without improving on the capabilities \citep{gudibande2023false}, amplify hallucinations when overfitting to the samples \citep{gekhman2024does}, and degrade previously learned capabilities through catastrophic forgetting \citep{luo2025empirical}. Our method studies the supervised-finetuning regime similar to standard SFT, but modifies the loss to avoid overfitting to samples and achieve the optimal point on the new-task/prior-capabilities tradeoff.

\vspace{1pt}
\noindent \textbf{Data and objective modifications.}
A line of work improves offline fine-tuning by modifying either the data or the objective. Data-level methods curate high-quality instructions \citep{zhou2023lima}, filter noisy examples \citep{chen2023alpagasus}, select for diversity \citep{liu2023makes}, use model-based filters \citep{li2024superfiltering,liu2024selectit}, or choose samples influential for a target capability \citep{xia2024less}. Other methods make the data more model-compatible by filtering low-likelihood examples or using self-generated data \citep{chen2025retaining,shenfeld2026self}. 

Objective-level methods instead keep the offline data fixed and change the loss weights. DFT \citep{wu2025generalization} is the closest to our work and rescales SFT by token likelihood. \citet{li2025beyond} also study likelihood-based supervised objectives. Similarly, InfoSFT reweights offline data, but rather than simply down-weighting low-likelihood samples, it derives the optimal weighting rule for learning from offline data and shows that middle-confidence tokens should receive the largest weights.

\vspace{1pt}
\noindent \textbf{On-policy learning.}
On-policy methods such as PPO, GRPO, RLOO, and RAFT are widely used in post-training and RLHF to align the model's distribution with human preferences or reward signals \citep{schulman2017proximal, shao2024deepseekmath,ahmadian2024back,dong2023raft, ouyang2022training}. Recent work suggests these methods often generalize better than standard SFT because they train closer to the model's own test-time distribution \citep{tajwar2024preference,chu2025sft,chen2025retaining}. However, they require rewards or verifiers, and become less effective when the base model cannot already produce the desired behavior or high reward solutions \citep{yue2025does,casper2023open,yoshihara2025practical}. We study the complementary regime where expert demonstrations remain necessary, but should be weighted adaptively to account for the model's current policy.

\vspace{1pt}
\noindent \textbf{Catastrophic forgetting and proximal updates.}
Catastrophic forgetting remains a fundamental challenge \citep{mccloskey1989catastrophic,ratcliff1990connectionist,kirkpatrick2017overcoming} and the primary obstacle for \textit{Continual learning} \citep{de2021continual,wang2024comprehensive}. In LLM post-training, narrow fine-tuning often improves target tasks while degrading general capabilities like safety or instruction-following \citep{scialom2022fine,kotha2023understanding,qi2023fine,huang2024mitigating,luo2025empirical}. Recent work connects forgetting to distributional drift: SFT forgets more than on-policy methods because offline updates induce a larger shift from the base model \citep{chen2025retaining,lai2025reinforcement}, where greater KL distances correlate with stronger forgetting  \citep{shenfeld2025rl,sabbaghi2026robust}. 
Consequently, proximal and mirror-descent frameworks constrain policy changes via the KL to a reference model \citep{schulman2015trust,schulman2017proximal,tomar2020mirror}. We adopt this proximal view, deriving a closed-form weighting rule that maximizes learning under a fixed KL budget.

\vspace{-5pt}
\section{Method}\label{sec:method}
We first provide an overview of SFT and DFT. We then demonstrate how these approaches can be unified under a general weighting function, and discuss how this function can be optimized to improve model performance.

\vspace{-5pt}
\paragraph{SFT.}
We study conditional generation where, given a prompt $x$, a response is a sequence of tokens
\(y=(y_1,\dots,y_{|y|})\), where
\(\pi_\theta(y|x)\!=\!\prod_t \pi_\theta(y_t| x,y_{<t})\).
We are given offline supervised data
\(\mathcal D=\{(x_{(i)},y^*_{(i)})
\}_{i=1}^N\), collected from an expert policy or a teacher model.
We first formulate objectives at the response level, with the final algorithm applying to token-wise weighting. Standard SFT maximizes the log-likelihood:
\begin{equation}
\label{eq:sft}
J_{\mathrm{SFT}}(\theta)
:=
\E_{(x,y^*)\sim \mathcal D}\big[\log \pi_\theta(y^*\mid x)\big]
\end{equation}
Its gradient can be reinterpreted as an expectation over samples from the model (on-policy):
\begin{equation}\label{eq:SFT_importance}
\nabla_\theta J_{\mathrm{SFT}}(\theta)
= \E_{(x,y^*)\sim \calD}\big[ \nabla_\theta \log \pi_\theta(y^*\mid x)\big]
=
\E_{(x,y^*)\sim \mathcal D,\, y\sim \pi_\theta(\cdot \mid x)}
\left[
\frac{\1(y = y^*)}{\pi_\theta(y^* \mid x)}
\nabla_\theta \log \pi_\theta(y \mid x)
\right]
\end{equation}
This is similar to the gradient in the policy gradient (PG) algorithm:
$\nabla_\theta J_{\mathrm{PG}}(\theta) \!=\! \E\big[
{r(x, y)} \nabla_\theta \log \pi_\theta(y | x)
\big]$
Thus, from the policy-gradient perspective, SFT implicitly assigns an effective reward of
\(1/\pi_\theta(y^*\mid x)\) to the expert sample $y^*$ and zero elsewhere. This reward coefficient increases for low-likelihood samples.

\vspace{-4pt}
\paragraph{DFT.}
DFT \citep{wu2025generalization} removes this inverse-likelihood factor by multiplying the supervised objective with the stop-gradient of the likelihood, where \(\mathrm{sg}(\cdot)\) denotes a quantity that is used in the forward pass but treated as constant during backpropagation:
\(
J_\mathrm{DFT}
:=
\E_{(x,y^*)\sim \mathcal D}
\big[\mathrm{sg}(\pi_\theta(y^*|x)) \log \pi_\theta(y^*|x)\big]
\). Its gradient becomes:
\begin{align}\label{eq:DFT_grad}
\nabla_\theta J_{\mathrm{DFT}}(\theta)
&=
\E_{(x,y^*)\sim \mathcal D}
\big[
\pi_\theta(y^*\mid x)\nabla_\theta \log \pi_\theta(y^* \mid x)
\big]\nonumber\\
&=\E_{(x,y^*)\sim \mathcal D,\, y\sim \pi_\theta(\cdot\mid x)}
\big[
\1(y = y^*) \nabla_\theta \log \pi_\theta(y \mid x)
\big]
\end{align}

\noindent Hence, in the supervised objective, DFT assigns smaller weights to low-likelihood samples and larger weights to those likely under the model. Equivalently, in the on-policy perspective according to ~\eqref{eq:DFT_grad}, DFT assigns an effective reward of 1 to expert samples. \citep{wu2025generalization} show that this choice improves training stability. However, it may also underweight samples that are informative for learning the new behavior.

\vspace{-4pt}
\paragraph{A general weighting framework.}
Both objectives can be unified through a general likelihood-dependent weighting function $\Omega(\pi_\theta(y^*| x))$ where $\Omega:[0, 1]\to \R$. 
Throughout, these weights are treated as gradient-free quantities, analogous to a reward function $r(x, y)$ in policy gradient. For brevity, we denote $q:= \pi_\theta(y^*| x)$. Since the weight carries no gradient, we can directly incorporate $\Omega(q)$ into the policy gradient formulation:
\begin{align}\label{eq:J_omega}
\nabla_\theta J_{\Omega}(\theta)
&=
\E_{(x,y^*)\sim \mathcal D,\, y\sim \pi_\theta(\cdot\mid x)}
\big[
\1(y = y^*)\Omega(q)
\nabla_\theta \log \pi_\theta(y \mid x)
\big] \nonumber \\
&=
\E_{(x,y^*)\sim \mathcal D}
\big[
q\Omega(q)\nabla_\theta \log \pi_\theta(y^* \mid x)
\big]
\end{align}
with the corresponding objective: \(J_\Omega := \E_{(x,y^*)\sim \mathcal D}\big[\mathrm{sg}(q\Omega(q)) \log \pi_\theta(y^*| x)\big]\). Thus, \(\Omega(q)\) can be interpreted as the effective reward assigned to the sample within the on-policy expectation. Crucially, the  weighting coefficient---multiplying the gradient of the observed sample---is \(q\Omega(q)\), as shown in the second line. Therefore, DFT corresponds to \(\Omega(q)=1\) and assigns weight \(q\) to the observed sample, while SFT corresponds to \(\Omega(q)=1/q\) yielding uniform weight $1$ across all samples.

To derive the weighting rule, \Cref{sec:theory} compares different choices of \(\Omega\) under a given KL distance budget from the base model.
The objective is to fit the unknown population distribution of expert samples \(p^*(\cdot\mid x)\), rather than overfitting to the single observed sample $y \sim p^*(\cdot \mid x)$. Thus the objective differs from SFT, which seeks to maximize the likelihood of the observed sample. Lemma \ref{lem:oracle} shows that, up to a global scaling factor absorbable into the learning rate, the weighting rule that best fits the population distribution takes the form:
\[
\Omega^*(q)
\propto
\mathrm{logit}\big(p^*(y^*\mid x)\big)-\mathrm{logit}(q),
\qquad
q=\pi_\theta(y^*\mid x), 
\ \
\mathrm{logit}(q):= \log\left(\frac{q}{1-q}\right)
\]
Intuitively, an oracle with access to the expert probabilities compares two confidences for each sample: if a sample is likely under the expert but unlikely under the model, it should be upweighted since it represents a behavior the model has not learned. In contrast, if a sample is more likely under the model than under the expert, it should not be further reinforced. Even correct samples may be relatively rare under the expert when multiple valid responses exist; thus, pushing every observed response toward probability one risks overfitting to individual samples rather than capturing the expert population distribution.

\vspace{-10pt}
\paragraph{InfoSFT.}
The oracle rule depends on the unknown expert probability \(p^*(y^*|x)\), 
so it cannot be implemented directly. We replace $p^*(y^*|x)$ with its average value $\bar p := \E_{x, y^*}[p^*(y^*|x)]$. Accordingly, the quantity $\mathrm{logit}(p^*(y^*|x))$ is approximated by the constant term $\mathrm{logit}(\bar p)$. This means that InfoSFT compares the model confidence with an average confidence level instead of the sample-specific expert confidence.
Lemma \ref{lem:expected} shows that our choice for the constant is near-optimal and yields a provable improvement in the population over both DFT and SFT. Following this approach, we set:
\[
w_{\mathrm{info}}(q) 
= q \Omega(q) = 
q(\mathrm{logit}(\bar p)-\mathrm{logit}(q))
\]

\noindent For clarity, the quantities used above play different roles. The function \(\Omega(q)\) is the coefficient in the on-policy view (first line of \Cref{eq:J_omega}), while
\(
w(q):=q\Omega(q)
\)
is the actual multiplier that appears behind gradients of samples in supervised training (second line of \Cref{eq:J_omega}).

We now apply this rule token-wise. Our results in \Cref{sec:theory} hold for token probabilities without any changes; the same weighting logic is applied to each conditional next-token prediction as individual samples, and the weighting rule applies to \(q_t \!= \! \pi_\theta(y_t^*|x,y^*_{<t})\). Note that the rule becomes negative when \(q_t\) is larger than $\bar p < 1$.
At the token level, there are many trivial tokens such as ``is'' or ``of'' that have $q \approx 1$, for which $w_{\mathrm{info}}(q) \to -\infty$. This would decrease the probability of such trivial tokens and harm fluency and answer quality. To avoid this, we clip the weighting rule to be positive. This clipping is a practical token-level heuristic: it leaves the theoretically derived positive-weight region unchanged, but turns negative updates into zero weight rather than penalizing already-confident tokens. By plugging the clipped rule into \Cref{eq:J_omega}, we obtain the final gradient (for brevity, we denote $\pi_\theta(y_t^*\mid x,y^*_{<t})$ with $\pi_\theta(y_t^*)$): 
\vspace{-2pt}
\begin{equation}\label{eq:InfoSFT_grad}
\boxed{
\nabla_\theta J_{\mathrm{InfoSFT}}
=
\E_{(x,y^*)\sim \mathcal D}
\left[
\frac{1}{|y^*|}
\sum_{t=1}^{|y^*|}
\underbrace{
\color{sftcolor}{\nabla_\theta \log \pi_\theta(y_t^*)}
}_{\text{\color{sftcolor}{1. Standard SFT}}}
\;
\underbrace{
\color{dftcolor}{\pi_\theta(y_t^*)}
}_{\text{\color{dftcolor}{2. DFT weight}}}
\;
\underbrace{
\color{infocolor}{
\Big[
\mathrm{logit}(\bar p)
-
\mathrm{logit}(\pi_\theta(y_t^*))
\Big]_+
}
}_{\text{\color{infocolor}{3. InfoSFT correction}}}
\right]
}
\end{equation}

The scalar \(\bar p\) is a calibration constant; in experiments, we estimate it from the model's average next-token confidence on its own generated responses. 
We study this experimentally in \Cref{sec:experiments}.

Note that InfoSFT has a middle-confidence weighting profile as shown in \Cref{fig:intro}. For small \(q\),
\(w_{\mathrm{InfoSFT}}(q)\) behaves like
\(q(\log(1/q))\), so relative to DFT this approach introduces an additional information-dependent factor while preserving the property that the weight vanishes as \(q\to0\). Moreover, for high-confidence tokens, the clipped correction term becomes zero, thereby avoiding unnecessary reinforcement of tokens that the model already predicts well.

\begin{remark}
     In \Cref{sec:entropy}, we show that InfoSFT is approximately equivalent to adding an entropy correction term to DFT. 
     This emphasizes expert trajectories currently unlikely under the model.
\end{remark}

\begin{remark}
InfoSFT is optimal under any KL budget from the base model. The effective budget in practice is determined by training hyperparameters such as the learning rate and the number of epochs.
\end{remark}

\section{Experiments}\label{sec:experiments}

We study two questions in our experiments: \textbf{(i)} whether InfoSFT improves test-time generalization after fine-tuning, and \textbf{(ii)} whether it better preserves prior capabilities, i.e., is less susceptible to catastrophic forgetting. For each dataset and model, we keep the learning rate, number of update steps, and all other training hyperparameters fixed across methods.

\vspace{-10pt}
\paragraph{Experiment overview.}
We first replicate the fine-tuning settings of \citet{wu2025generalization} on several models and datasets to evaluate generalization at test-time benchmarks. In the same setting, we also study how pass@\(k\) and output entropy behave across methods. We then evaluate InfoSFT for chain-of-thought (CoT) fine-tuning using reasoning traces from DeepSeek \citep{openr1math220k,guo2025deepseek}. Finally, we replicate the setting of \citet{shenfeld2026self} using their custom datasets, which are designed to measure the tradeoff between learning a new task and preserving prior capabilities.

\vspace{-10pt}
\paragraph{Baselines.}
Our primary baselines are standard SFT and DFT~\citep{wu2025generalization}, corresponding to uniform weighting and likelihood-proportional weighting, respectively. For each dataset, all methods are trained under the same hyperparameter setting.

\vspace{-10pt}
\paragraph{Hyperparameter selection.}\label{sec:pbar}
As discussed in \Cref{sec:method}, InfoSFT replaces the point-wise expert probability \(p^*\) with its average, following Lemma~\ref{lem:expected}. Since we do not assume access to teacher logits or the underlying expert distribution, we estimate this average using the student model's token probabilities, conditioned on correct student responses. For a range of models and datasets, we measure the average token probability at temperature \(0.7\)  (matching our evaluation settings). As shown in \Cref{table:pbar}, all estimates fall in the narrow range \([0.9,0.95]\), suggesting that this quantity is stable across models and tasks. We therefore sweep \(\bar p\) around this interval  to further investigate the optimal value for this parameter. We report pass@1 for each value in \Cref{fig:pbar_realtive}. Although the best value sometimes shifts between \(\bar p=0.9\) and \(\bar p=0.95\), \uline{\(\bar p=0.93\) is consistently close to the peak. We thus use \(\bar p=0.93\) as the default for all experiments.}

\subsection{Test-time Performance of InfoSFT}\label{sec:performance}

We evaluate InfoSFT on Qwen-2.5-Math-1.5B, Qwen-2.5-Math-7B, and Llama-3.1-8B \citep{yang2024qwen2,yang2024qwen-math,grattafiori2024llama} to cover different model families and sizes. For math fine-tuning, we train on 100K samples from ``NuminaMath-CoT'' \citep{li2024numinamath}. For code fine-tuning, we use 12K code-related samples from ``UltraFeedback'' \citep{cui2023ultrafeedback}, selecting the highest-scored responses. In both settings, we train for one epoch, following \citet{wu2025generalization}. The learning rate is \(5\times 10^{-5}\) for Qwen-2.5-Math-1.5B, and \(1\times 10^{-5}\) for Qwen-2.5-Math-7B and Llama-3.1-8B. For code fine-tuning on UltraFeedback, we use LoRA with rank \(32\), \(\alpha=64\), and learning rate \(2\times 10^{-5}\) for all models.

For math evaluation, we use ``MATH500'' \citep{hendrycks2021measuring} and ``AIMO-Validation-AMC'' \citep{aimovalidationamc}, which contains 83 problems from AMC 12 2022 and 2023. For coding, we evaluate on ``HumanEval'' and ``MultiPL-E'' \citep{chen2021evaluating,cassano2022multipl}.

\paragraph{Main results.}
We report the results for both math and code fine-tuning in \Cref{table:main_results}. On math, InfoSFT consistently outperforms SFT and DFT across all three base models and all reported MATH500 and AMC metrics. We also observe stronger resutls on AMC: on the Qwen models, SFT and DFT often improve MATH500 while decreasing AMC performance, whereas InfoSFT improves the performace on AMC by 6 points for Qwen-1.5B. Notably, on Qwen-2.5-Math-7B, SFT gives only a marginal gain over the base model on MATH500 acc@1, while InfoSFT improves it substantially. InfoSFT also yields stronger pass@8 results than the baselines, which is important for later stages such as online RL or best-of-\(n\) sampling, where performance depends on producing at least one high-quality response among multiple samples.

For code, InfoSFT gives comparable or better results than DFT. On the Qwen models, DFT and InfoSFT achieve similar performance. On Llama-3.1-8B, InfoSFT performs best, with about a 3-point advantage over DFT on HumanEval. Overall, these results show that InfoSFT is the most consistently beneficial method across the fine-tuning settings we evaluate.

We present training token accuracy plots in \Cref{sec:tok_accuracy}. Although SFT reaches higher training token accuracy, InfoSFT achieves better test-time performance. This indicates that fitting the training tokens more closely does not necessarily lead to better generalization.

\begin{table*}[t!]
  \centering
  \label{table:main_results}
  \vspace{-4pt}
  \resizebox{\columnwidth}{!}{%
    \begin{tabular}{@{}lcccccccc@{}}
      \toprule
      & \multicolumn{4}{c}{Math Training on Numina-Math} 
      & \multicolumn{4}{c}{Instruction fine-tuning on UltraFeedback} \\
      & \multicolumn{2}{c}{MATH500} 
      & \multicolumn{2}{c}{AMC} & \multicolumn{3}{c}{HumanEval} & \multicolumn{1}{c}{MultiPL-E} \\
      \cmidrule(lr){2-3} \cmidrule(lr){4-5} \cmidrule(lr){6-8} \cmidrule(lr){9-9} 
      Models & acc@1 & pass@8 & acc@1 & pass@8 & HE &  HE+ & pass@8 & Avg.(8 lang) \\
      \midrule
      Qwen-Math-1.5B (base)& 33.2 & 77.5 & 34.9 & 63.4 & 40.9 & 35.4 & 56.7 & 26.0 \\
      Qwen-Math-1.5B (\textbf{SFT}) & 61.6 & 81.1 & 31.3 & 54.4 & 43.0 & 36.6 & 58.8 & 28.0 \\
      Qwen-Math-1.5B (\textbf{DFT}) & 59.2 & 77.0 & 30.1 & 56.4 & \textbf{46.7} & 40.3 & \textbf{58.9} & \textbf{29.4} \\
      Qwen-Math-1.5B (\textbf{InfoSFT})& \textbf{66.2} & \textbf{84.3} & \textbf{41.0} & \textbf{67.5} & \textbf{46.7} & \textbf{41.2} & 58.8 &  28.6 \\
      \midrule
      Qwen-Math-7B  (base) & 52.8 & 82.9 & 39.8 & 68.0 & 65.2 & 59.8 & 71.5 & 30.2\\
      Qwen-Math-7B (\textbf{SFT}) & 53.4 & 84.4 & 39.9 & 62.1 & 66.5 & 59.4 & 71.6 & 28.4\\
      Qwen-Math-7B (\textbf{DFT}) & 65.4 & 83.2 & 34.1 & 58.9 & \textbf{69.5} & \textbf{64.0} & 70.3 & \textbf{35.2}\\
      Qwen-Math-7B (\textbf{InfoSFT}) & \textbf{69.7} & \textbf{87.2} & \textbf{43.4} & \textbf{69.1} & 68.9 & \textbf{64.0} & \textbf{71.9} & 34.1 \\
      \midrule
      Llama-3.1-8B  (base) & 3.2 & 20.5 & 0.9 & 6.8 & 38.4 & 31.1 & 66.0 &  31.0 \\
      Llama-3.1-8B (\textbf{SFT}) & 24.0 & 57.4 & 8.4 & 31.7 & 40.9 & 34.5 & \textbf{66.3} & 32.0\\
      Llama-3.1-8B (\textbf{DFT}) & 15.5 & 29.6 & 6.0 & 15.2 & 41.8 & 35.0 & 54.4 & \textbf{32.6}\\
      Llama-3.1-8B (\textbf{InfoSFT}) & \textbf{27.8} & \textbf{58.8} & \textbf{13.2} & \textbf{35.2} & \textbf{44.8} & \textbf{37.8} & 64.6 & 32.2 \\
      \bottomrule
    \end{tabular}
  }
  \vspace{-6pt}
  \caption{Comparison of SFT, DFT, and InfoSFT across math and code.  \uline{All the reported numbers are averaged over 3 seeds.} InfoSFT has dominant math performance across all base models, improving both acc@1 (temp=0) and pass@8 (temp=0.7). InfoSFT achieves competitive results on code-generation benchmarks as well. The pass@8 gains suggest that InfoSFT improves the learned policy beyond greedy decoding.}
  \vspace{-5pt}
\end{table*}

\paragraph{SFT underperforms at pass@1 but preserves diversity.}
Prior studies \citep{wu2025generalization, shenfeld2026self} compare SFT with methods like DFT at pass@1 to showcase the shortcoming of SFT. Our results for Qwen-Math (\Cref{fig:pass@k_entropy}) confirm that SFT has a lower pass@1 at T=0.7, but this does not mean that SFT learns less; SFT outperforms DFT at greedy decoding (acc@1, T=0) and in higher pass@k metrics as shown in \Cref{table:main_results}. This difference is because DFT entropy converges to zero (\Cref{fig:pass@k_entropy}, right), resulting in deterministic outputs, while SFT maintains high entropy. 
Consequently, DFT improves pass@1 performance because it is less susceptible to sampling noise at higher temperatures. However, it degrades pass@k performance (\Cref{fig:pass@k_entropy}, left) by disproportionately favoring high-likelihood samples. This neglects less probable, yet correct, responses that provide valuable diversity when evaluating multiple samples. As previously discussed, strong pass@k results are critical for stages after supervised fine-tuning.

In contrast, \Cref{fig:pass@k_entropy} shows that InfoSFT achieves the best pass@\(k\) performance. Notably, it also avoids the entropy collapse observed for DFT, preventing the model's responses from becoming overly deterministic.

\begin{figure*}
    \centering
    \begin{subfigure}{0.333\linewidth}
        \includegraphics[width=\linewidth]{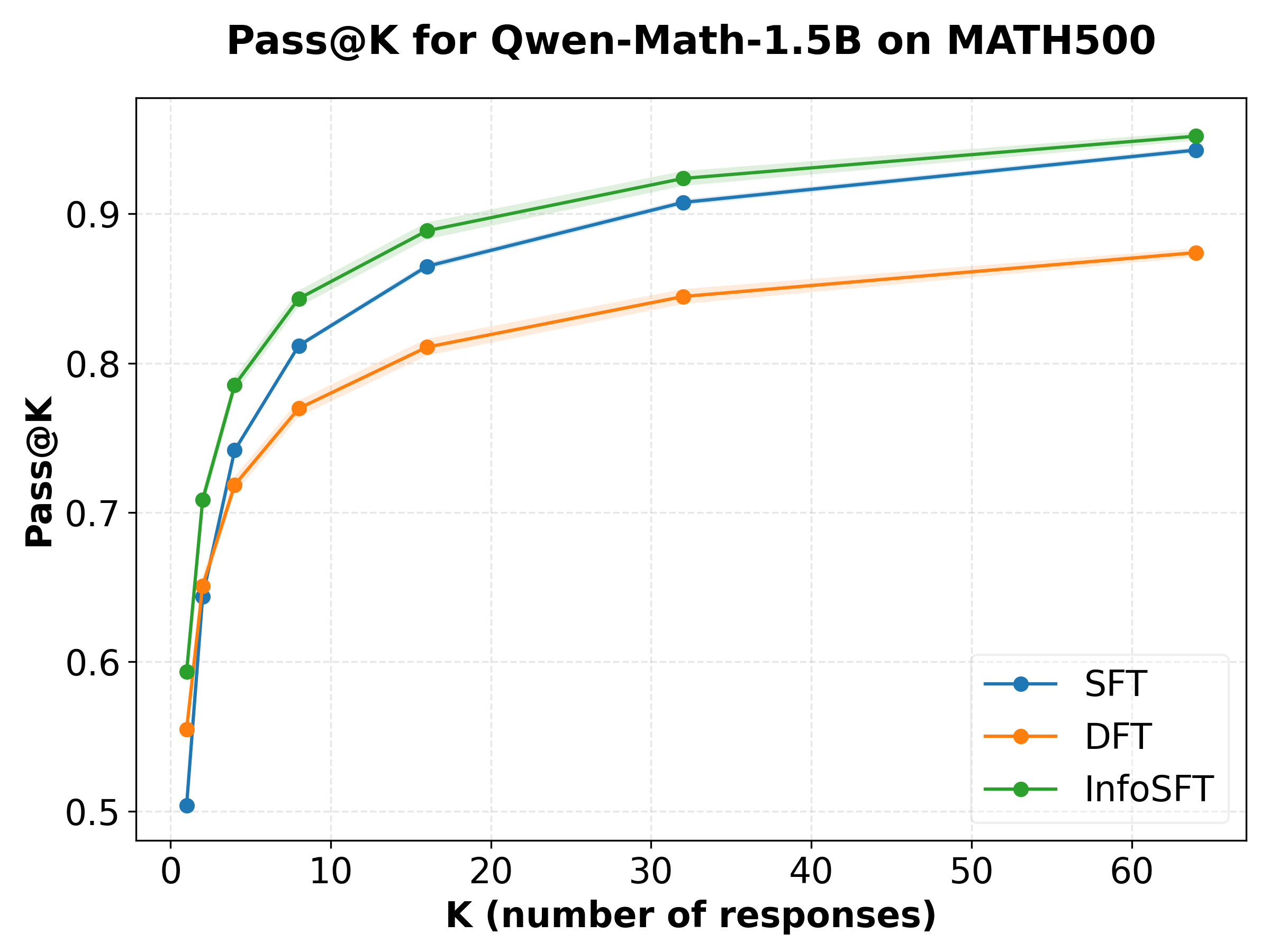}
    \end{subfigure}%
    \begin{subfigure}{0.333\linewidth}
        \includegraphics[width=\linewidth]{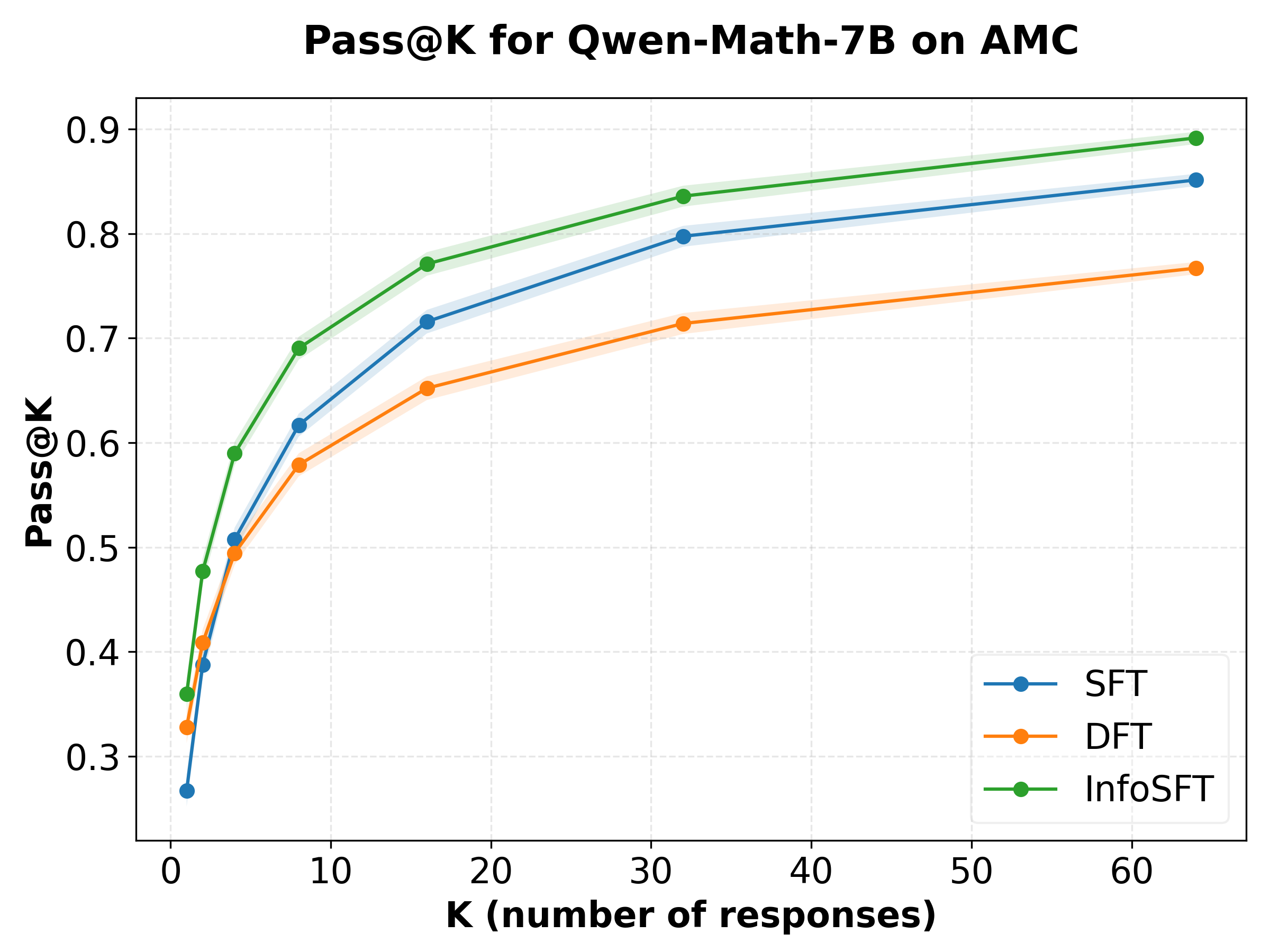}
    \end{subfigure}%
    \begin{subfigure}{0.333\linewidth}
        \includegraphics[width=\linewidth]{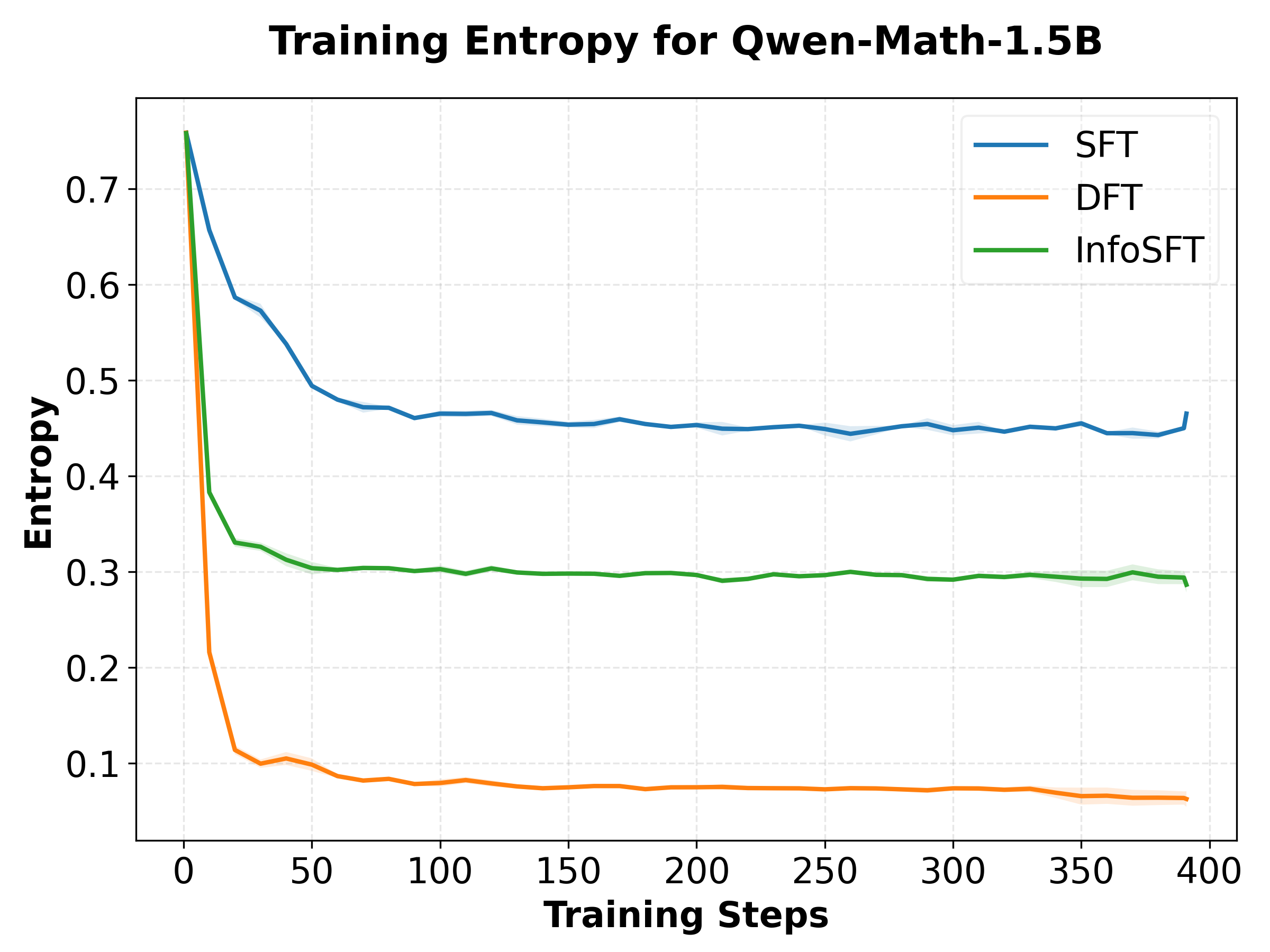}
    \end{subfigure}
    \vspace{-15pt}
     \caption{\textbf{(left/middle)} Show pass@k with $k\in[1, 64]$ for Qwen-Math-1.5B and Qwen-Math-7B trained with the three methods. InfoSFT is higher than other baselines for all k. \textbf{(right)} Unlike DFT, InfoSFT controls the entropy and avoids mode collapse. This randomness is crucial for any later training or alignment of the model.}
     \vspace{-3pt}
     \label{fig:pass@k_entropy}
\end{figure*}

\paragraph{SFT and InfoSFT are complementary for difficult samples.}
In addition to the results in \Cref{table:main_results}, we study CoT fine-tuning on reasoning data from ``OpenR1-Math'' \citep{openr1math220k}, generated by DeepSeek-R1 \citep{guo2025deepseek}. The responses follow the format
\[
\texttt{<think> reasoning trace </think> answer}.
\]
This setting requires the model to learn a new reasoning format that is initially unlikely under the model, making it challenging for methods that favor already likely tokens. We train Qwen-2.5-7B-Instruct with \( \mathrm{lr} = 5\times 10^{-6}\) for two epochs on 70K samples from the default OpenR1 subset with completion\_length \(\leq 8192\). We evaluate on AMC and AIME24 \citep{aimovalidationamc,aimovalidationaime} with a \(16192\)-token generation limit.

\Cref{fig:openR1} (left) shows that SFT outperforms InfoSFT and DFT in this setting since SFT does not downweight the unlikely tokens needed to learn the thinking format. Nevertheless, InfoSFT still improves over the base model on AMC while remaining close to the base model's style (see \Cref{sec:examples}). Motivated by this observation, we first train the model with SFT for one epoch to increase the likelihood of tokens that are unlikely under the base model, such as \texttt{<think>} and \texttt{</think>}, and then continue training the resulting model for one additional epoch with InfoSFT. This two-stage procedure yields the best AIME results, especially in pass@8, while two epochs of plain SFT remain best on AMC. Moreover, \Cref{fig:openR1} (right) shows that even 50 steps of InfoSFT after the initial SFT stage improve AIME by 3.5 points. This supports the view that the two objectives are complementary: SFT is effective for introducing very unlikely formats, while InfoSFT focuses training on middle-confidence tokens once those formats become more likely under the model.

\begin{figure}[t]
    \centering

    \begin{subfigure}[c]{0.57\linewidth},
        \centering
        \resizebox{\linewidth}{!}{%
        \begin{tabular}{@{}lcccc@{}}
            \toprule
            & \multicolumn{2}{c}{AMC} & \multicolumn{2}{c}{AIME} \\
            \cmidrule(lr){2-3} \cmidrule(lr){4-5}
            Method & acc@1 & pass@8 & acc@1 & pass@8 \\
            \midrule
            Base (Qwen-2.5-7B-Instruct) & 43.7 & 70.5 & 16.3 & 33.3 \\
            SFT-epoch1 & 49.2 & 75.9 & 14.6 & 40.0 \\
            SFT-epoch2 & \textbf{52.4} & \textbf{80.7} & 20.0 & 43.3 \\
            DFT & 38.1 & 62.7 & 8.8 & 16.7 \\
            InfoSFT & 47.5 & 73.5 & 13.8 & 26.7 \\
            SFT-epoch1 + DFT(1 epoch)& 33.6 & 59.0 & 12.5 & 30.0 \\
            SFT-epoch1 + InfoSFT(1 epoch)& 50.9 & 76.4 & \textbf{20.8} & \textbf{50.0} \\
            \bottomrule
        \end{tabular}%
        }
    \end{subfigure}
    \hspace{0.15in}
    \begin{subfigure}[c]{0.33\linewidth}
        \centering
        \captionsetup{skip=1pt}
        \includegraphics[width=\linewidth]{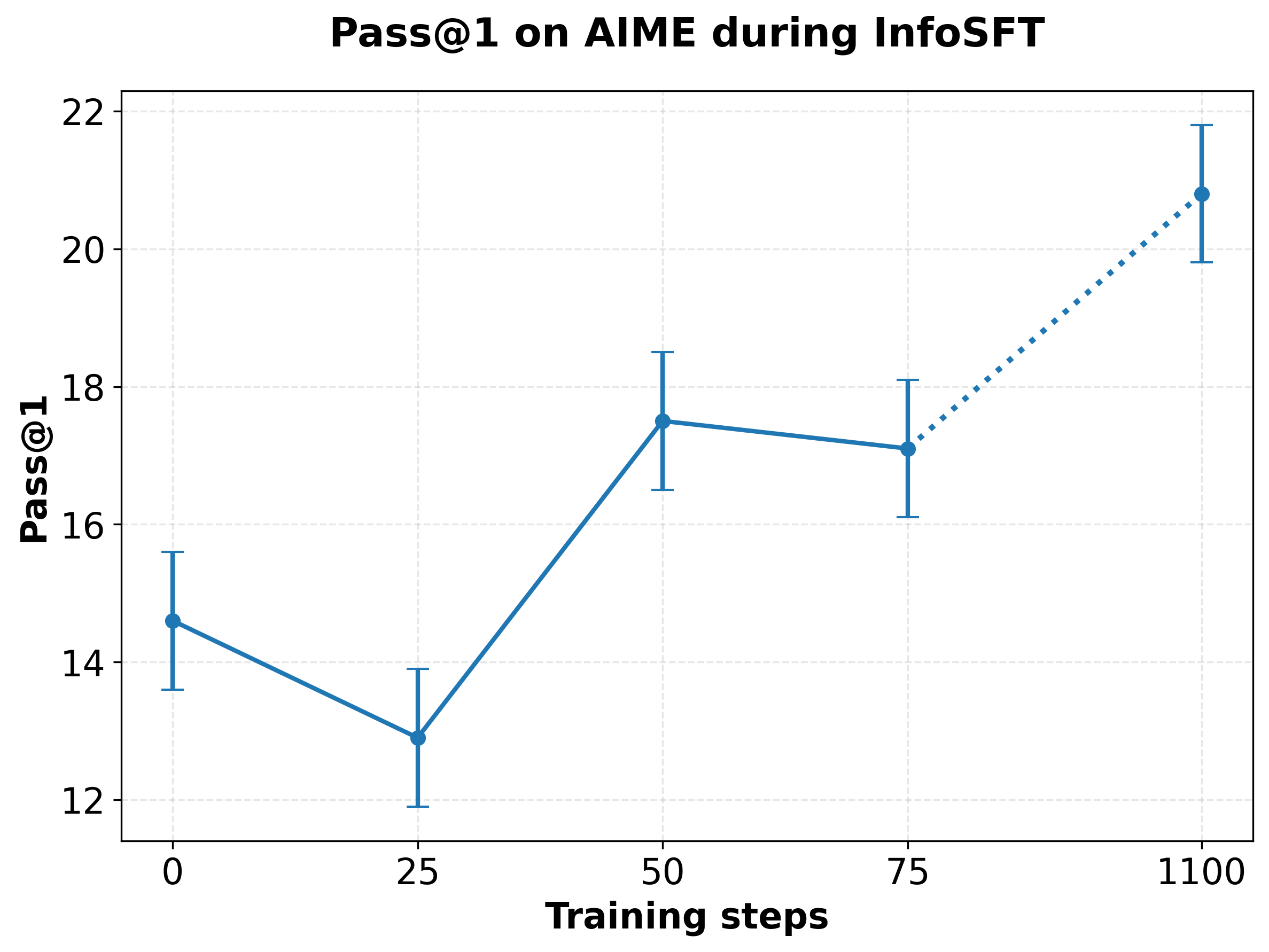}
    \end{subfigure}

    \vspace{-6pt}
    \caption{\textbf{(left)} On reasoning samples from OpenR1 \citep{openr1math220k} with low likelihood under the base model, SFT does better since both DFT and InfoSFT down-weight the unlikely tokens such as ``<think>'' and do not learn the thinking format. However, unlike DFT, InfoSFT still improves over the base model while keeping the base model's format. However, applying 1 epoch of InfoSFT after the first epoch of SFT (which boosts the unlikely tokens) achieves the best results on AIME. \textbf{(right)} Shows that even 50 steps of InfoSFT after the initial SFT stage boosts the performance on AIME.}
    \label{fig:openR1}
    \vspace{-4pt}
\end{figure}

\vspace{10pt}

\subsection{Avoiding Catastrophic Forgetting}\label{sec:tradeoff}

The InfoSFT objective finds the optimal weight for the tradeoff between learning the new task against staying close to the base model (see \Cref{eq:star}). We therefore test whether InfoSFT gives a better learning-forgetting tradeoff in practice. Following \citet{shenfeld2026self}, we fine-tune Qwen-2.5-7B-Instruct on two new tasks: Science Q\&A, using 2.7K Chemistry L-3 samples from SciKnowEval \citep{feng2024sciknoweval}, and Tool Use, using 4K samples from ToolAlpaca \citep{tang2023toolalpaca}. For each task, we train SFT, DFT, and InfoSFT under multiple learning rates and for either 1 or 2 epochs. For Science Q\&A, we sweep $\mathrm{lr}\in\{1\mathrm{e}{-6},2\mathrm{e}{-6},5\mathrm{e}{-6},7\mathrm{e}{-6},1\mathrm{e}{-5}\}$, giving 10 checkpoints per method. For Tool Use, we use $\mathrm{lr}\in\{1\mathrm{e}{-6},2\mathrm{e}{-6},5\mathrm{e}{-6},7\mathrm{e}{-6}\}$, giving 8 checkpoints per method; we omit $\mathrm{1e}{-5}$ because models overfit this task more quickly. For each checkpoint, we report new-task accuracy together with prior-capability score, measured as the average accuracy on HumanEval, IFEval \citep{zhou2023instruction}, and MATH-500.
This gives a tradeoff curve showing how each method learns the new task as training becomes more aggressive.

\begin{wrapfigure}{r}{0.48\linewidth}
  \vspace{-8pt}
  \centering
  \includegraphics[width=1.0\linewidth]{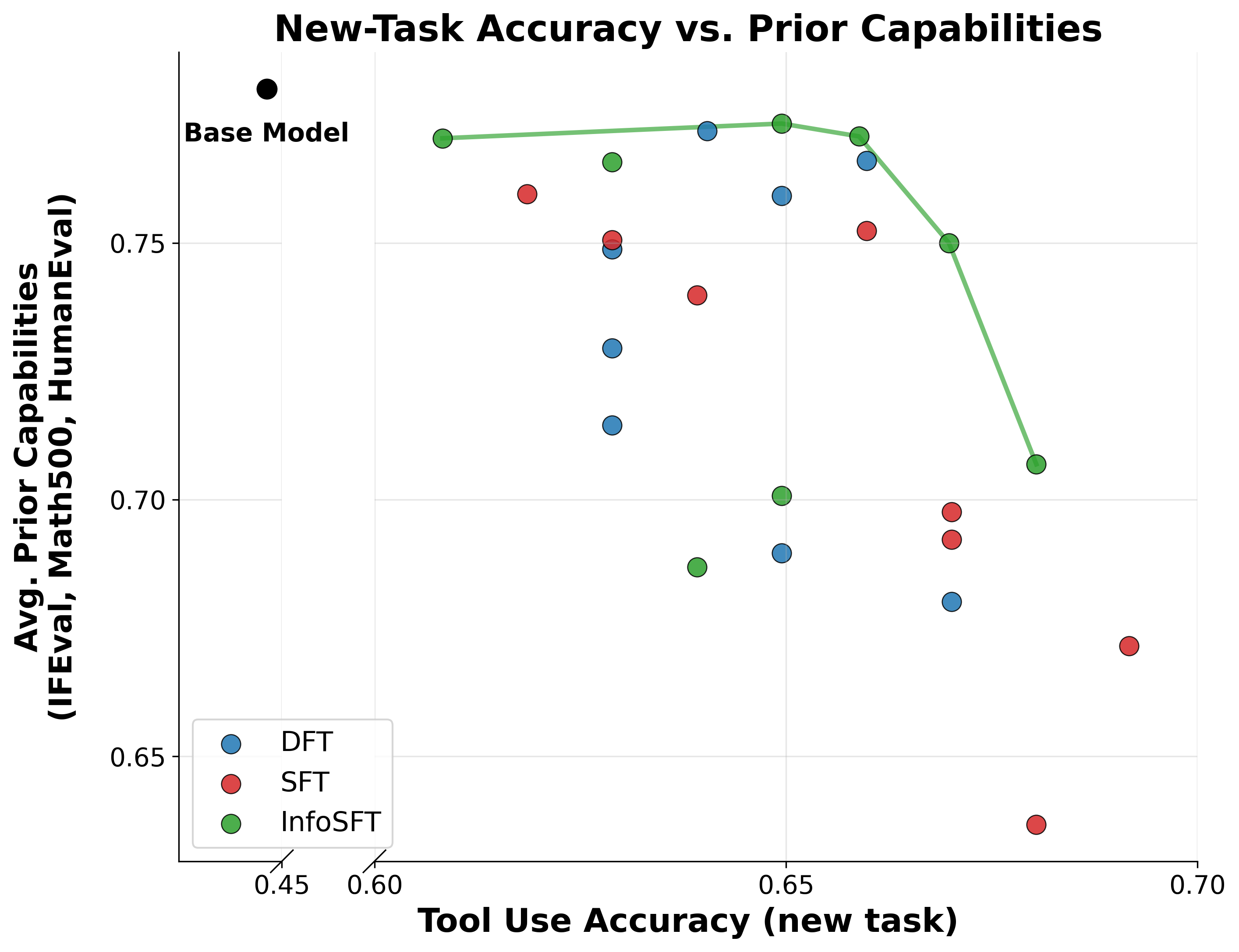}
  \vspace{-9pt}
  \caption{Tool-use performance vs. prior capabilities across 8 $(\mathrm{lr}, \#\mathrm{epochs})$ configurations. InfoSFT achieves a better learning-forgetting tradeoff compared to SFT and DFT.}
  \label{fig:tool_tradeoff}
  \vspace{-15pt}
\end{wrapfigure}
\Cref{fig:intro} (right) shows the Science Q\&A results. InfoSFT achieves the best tradeoff curve: it reaches higher new-task accuracy than SFT and DFT while preserving prior-capability. The best checkpoint of InfoSFT, also achieves the highest Science Q\&A accuracy. 
\Cref{fig:tool_tradeoff} shows the same tradeoff for Tool Use. Forgetting is stronger in this setting because samples follow a specialized format that is farther from the base model distribution. Here, SFT achieves a better performance on the new task: it learns the new task more aggressively, but with larger degradation of prior capabilities compared to InfoSFT. InfoSFT again lies on a better tradeoff curve, improving tool-use accuracy while preserving the model's previous capabilities. Note that individual checkpoints can still learn more and forget more for different methods (see \Cref{sec:all_numbers} for all the numbers), so we compare the full sweep rather than a single hyperparameter setting.

\section{Theory: Optimality of InfoSFT}\label{sec:theory}

\textbf{Warmup.}
\Cref{sec:method} describes SFT, DFT, and InfoSFT through gradients under the policy-gradient framework. To compare weighting rules here, we use the proximal update framework from RL
\citep{tomar2020mirror,schulman2015trust}, following an analysis similar to
\citet{mroueh2025reinforcement}. This framework is used only as an analysis tool: instead of tracking parameter gradients for \(\theta\), it directly characterizes the updated distribution \(\pi\) obtained after rewarding the observed demonstration while staying close to the base model \(\pi_0\). Therefore, similar to the approach of DPO \citep{rafailov2023direct}, this allows us to characterize the optimal policy $\pi$ directly in the space of distributions. Specifically, given a reward function \(R(x,y)\), the proximal update is:
\begin{equation}\label{eq:proximal}
\pi(\cdot \mid x) = \argmax_{\pi(\cdot\mid x)} \Big\{
  \E_{y \sim \pi(\cdot\mid x)}\!\big[R(x, y)\big]
  - \beta\,\KL\big(\pi(\cdot\mid x) \| \pi_{0}(\cdot\mid x)\big)
\Big\}.
\end{equation}
Here \(R\) is evaluated before the update and then held fixed while optimizing over \(\pi\). The parameter \(\beta\) controls the movement budget: larger \(\beta\) keeps \(\pi\) closer to \(\pi_0\). The closed-form optimizer is the Gibbs distribution:
\[
\pi (y\mid x) \propto \pi_{0}(y\mid x)\exp\!\big(R(x,y)/\beta\big)
\]
To connect this tool to \Cref{sec:method}, consider a sample \(y^*\) and let \(q=\pi_0(y^*| x)\). A likelihood-dependent coefficient \(\Omega(q)\) corresponds in the proximal view to a general reward function: 
\[
R_\Omega(x,y)=\Omega(q)\1\{y=y^*\}.
\]
Thus, DFT corresponds to \(\Omega(q)=1\), while SFT corresponds to \(\Omega(q)=1/q\). These are the same coefficients that appear in \Cref{eq:J_omega}. The only difference is the viewpoint: \Cref{sec:method} applies the stop-gradient weighting rule to the gradients, whereas the proximal framework applies the same coefficients and solves directly for the updated distribution. Additionally, while the likelihood q is calculated under the base model in our one-step analysis here, the gradient implementation in \Cref{sec:method} applies the same rule at the current model before the update, with likelihood-dependent weights as stop-gradient quantities.

The KL term provides a controlled way to compare weighting rules under the same amount of divergence from the base model. This is motivated by the empirical observation that larger KL divergences from the base model are closely linked to more catastrophic forgetting \citep{shenfeld2025rl,sabbaghi2026robust}. So the question we aim to answer is: \textit{``Under a controlled amount of catastrophic forgetting, what update rule learns more?''}.
We compare methods under any given update budget---which is determined by an arbitrary value of $\beta$ here, and in practice is determined by hyper-parameters such learning rate and number of epochs.

\subsection{Problem Formulation}

Fix a prompt \(x\), and let \(y^* \sim p^*(\cdot \mid x)\), where \(p^*\) is the unknown expert distribution. We state the analysis for a full response \(y^*\) to keep notation simple; the token-level rule used by InfoSFT follows by applying the same argument to each conditional next-token distribution \(\pi(\cdot\mid x,y^*_{<t})\).
Let
\[
q := \pi_{0}(y^* \mid x),
\qquad
p := p^*(y^* \mid x).
\]
We consider a general weighting function \(\Omega(q)\). Since \(q\) is computed under the base model and thus before the update, \(\Omega(q)\) is fixed during the proximal update, matching the stop-gradient treatment of the weights in \Cref{sec:method}. The corresponding proximal update is:
\begin{equation}
\label{eq:star}
\pi (\cdot\mid x)
=
\argmax_{\pi(\cdot\mid x)}
\Big\{
\E_{y\sim \pi(\cdot\mid x)}\big[\Omega(q) \,\1\{y=y^*\}\big]
-
\beta\, \KL \big(\pi(\cdot\mid x) \| \pi_{0}(\cdot\mid x)\big)
\Big\}.
\end{equation}
\vspace{-10pt}

\noindent The standard variational argument yields the closed-form solution to \Cref{eq:star}:
\begin{equation}\label{eq:gibbs}
\pi (y \mid x)
=
\frac{\pi_{0}(y \mid x)}{Z}
\exp\!\Bigl(\frac{\Omega(q)}{\beta}\,\1\{y=y^*\}\Bigr),
\qquad
Z := q\, e^{\Omega(q)/\beta} + 1 - q .
\end{equation}
\vspace{-6pt}

\noindent The partition function \(Z\) ensures normalization. Since \(y^*\) is only a single draw from the expert \(p^*(\cdot \mid x)\), in general \(p^*(y^* \mid x)<1\): there can be many correct responses for the same prompt. The right measure of progress is therefore not the likelihood of this one sample, whose maximization leads to sample overfitting as explained before, but the population KL divergence:
\[
\KL\!\bigl(p^*(\cdot \mid x)\,\|\,\pi(\cdot \mid x)\bigr)
\]

\noindent Thus, the choice of \(\Omega\) is determined by how much the update reduces this population KL.

\begin{lemma}[Oracle rule for \(\Omega\)]\label{lem:oracle}
Fix a prompt \(x\), response \(y^*\), and let \(q := \pi_{0}(y^* \mid x)\) and \(p := p^*(y^* \mid x)\). Write \(u := \Omega(q)/\beta\). Then:
\begin{enumerate}[label=(\alph*), wide, labelindent=0pt,topsep=2pt,itemsep=1pt]
\item The change in population KL after one step according to \Cref{eq:star} is
\[
\Delta\KL
=
\KL\bigl(p^*(\cdot \mid x)\| \pi (\cdot \mid x)\bigr)
-
\KL\bigl(p^*(\cdot \mid x)\| \pi_{0}(\cdot \mid x)\bigr)
=
\log Z - p\,u.
\]

\item \(\Delta\KL\) is strictly convex in~\(u\):
$
\frac{\partial^2\Delta\KL}{\partial u^2}
=
\pi (y^* | x)(1-\pi (y^* | x))>0
$.
The unique minimizer is:
\begin{equation}\label{eq:oracle-u}
u^*
= \frac{\Omega^*}{\beta} =
\mathrm{logit}(p) - \mathrm{logit}(q),
\qquad
\mathrm{logit}(t):=\log\!\frac{t}{1-t}.
\end{equation}
\end{enumerate}
\end{lemma}

The oracle update sets the new probability of the observed response to its expert probability. Since \(p\) is unknown, this rule cannot be implemented directly. We therefore replace the unknown term \(\mathrm{logit}(p)\) with a constant, giving the family: $u_C(q)=C-\mathrm{logit}(q)$.

\subsection{Near-optimality for the Expected KL}

Since \(\Omega\), and equivalently \(u=\Omega/\beta\), is only a function of \(q\), the expected test KL decomposes by conditioning on \(q\) (note that $q = \pi_{0}(y^*|x)$ is a random variable across prompts):
\begin{equation}\label{eq:expected-kl}
\E_{x, y^*}[\Delta\KL]
=
\E_q\!\bigl[\log(q e^{u(q)} + 1 - q) - \bar{p}(q)\, u(q)\bigr],
\qquad
\bar{p}(q) := \E[p \mid \pi_{0}(y^* \mid x) = q].
\end{equation}
The function \(\bar p(q)\) is still unknown because it depends on the expert distribution. The family \(u_C(q)=C-\mathrm{logit}(q)\) approximates this unknown calibration by a constant shift. The next lemma shows that this family has an optimal shift \(C^*\), and that the resulting rule is near-oracle and strictly better than DFT and SFT in expected KL reduction.

\begin{lemma}[InfoSFT dominates DFT and is near-oracle]\label{lem:expected}
Let \(u_C(q)=C-\mathrm{logit}(q)\), assume \(q \le d\), write \(\bar p:=\E_{x, y^*}[p] = \E_q[\bar p(q)]\), and assume that \(d \leq \bar{p}/e^2\). Then:
\begin{enumerate}[label=(\alph*),wide,labelindent=0pt,topsep=2pt,itemsep=1pt]
\item \(\E[\Delta\KL(u_C)]\) is strictly convex in \(C\), and its unique minimizer \(C^*\) satisfies
$
C^* = \mathrm{logit}(\bar p) + O(d)
$
when \(d \ll 1\).

\item
Define
$
u_{\mathrm{info}}(q):=\mathrm{logit}(\bar p)- \mathrm{logit}(q).
$
Then 
$
\E[\Delta\KL_{\mathrm{info}}]-\E[\Delta\KL^*]
=
H_b(\bar{p})-\E[H_b(p)] \ge 0,
$
where \(H_b(t)\) is the binary entropy. Equality holds when \(p\) is constant. Moreover, when \(d \ll \bar p\):
\[
\frac{\E[\Delta \KL(u_{\mathrm{info}})]}{\E[\Delta \KL(u^*)]}
\geq
1 - \mathcal O\!\left( \frac{ |\max\{\log(1 / \bar p), 1\}|}{\log (1/d)} \right).
\]

\item InfoSFT strictly improves over DFT in KL reduction under the same budget. Furthermore, if $\bar p \leq 0.98$, then InfoSFT strictly improves over SFT.
\end{enumerate}
\end{lemma}

\begin{remark}
The condition \(q \leq \bar{p}/e^2\) captures the supervised fine-tuning regime in which the model assigns smaller likelihood to expert responses than the expert distribution does on average.
\end{remark}

\begin{remark}\label{remark:beta}
The parameter \(\beta\) only appears through \(u=\Omega/\beta\). Once the shape of \(u(q)\) is fixed, changing \(\beta\) rescales the objective by a global constant, which can be absorbed into the learning rate in the gradient form of the update in \Cref{eq:J_omega}. Thus, the analysis identifies the relative weighting rule across likelihoods for any fixed amount of movement from the base model. Stopping earlier or later changes this movement budget, but not the derived shape of the weighting rule.
\end{remark}

Finally, returning to the supervised gradient in \Cref{eq:J_omega}, the weighting rule on an observed token is \(q\Omega(q)\). Absorbing $\beta$ into the learning rate and clipping negative weights gives the InfoSFT token weight:
\vspace{-10pt}
\[
w_{\mathrm{InfoSFT}}(q)
\propto
q\Big[\mathrm{logit}(\bar p)-\mathrm{logit}(q)\Big]_+,
\qquad
q=\pi_\theta(y_t^*\mid x,y^*_{<t})
\]
\vspace{-10pt}
\section{Conclusion}
We revisit supervised fine-tuning as a crucial stage for teaching LLMs new behaviors from offline expert data. We show that both uniform fitting and likelihood-proportional weighting miss an important tradeoff between learning low-likelihood expert tokens and preserving prior capabilities. We derive InfoSFT, a simple token-weighting rule that emphasizes informative middle-confidence tokens. Across math, code, reasoning, InfoSFT improves generalization and achieves a better learning-forgetting tradeoff. These results highlight token weighting as a key ingredient for reliable supervised post-training.

\section*{Acknowledgment} 
This research has been supported by Coefficient Giving and the UK AI Security Institute. AJ was supported in part by  the Sloan fellowship in mathematics, the NSF Award DMS-2311024, an Amazon
Faculty Research Award, an Adobe Faculty Research Award and an iORB grant form USC Marshall School of Business.

\newpage

\bibliography{bibliography}
\bibliographystyle{unsrtnat}

\newpage
\appendix

\section{Additional Experiments}\label[appendixsection]{sec:additional}
\subsection{Average Token Probability}\label[appendixsection]{sec:tok_prob}
For a subset of $N = 100$ questions from each Numina-Math-CoT dataset and OpenCodeInstruct, we generated the responses with several models and read the probabilities of the tokens. We then took the average over tokens: 

$$\bar p _{\text{student}} = \frac{1}{NL} \sum_{x_{(i)}, y \sim \pi_\theta(\cdot|x_{(i)}) } \pi_\theta(y_t|[x_{(i)}, y_{<t}])$$

\noindent This value is reported in \Cref{table:pbar}. As it can be seen, $\bar p _{\text{student}}$ is consistently in $[0.9, 0.95]$. This is completely in agreement with our results in \Cref{sec:pbar} where we consistently achieve the beast performance (acc@1 and pass@8) for the choice of  $\bar p = 0.93$ as the hyper-parameter in \Cref{eq:InfoSFT_grad}.  

\begin{table}[h!]
\centering
\caption{The average token probability of the student model on the its correct responses is consistently in the interval of $[0.9, 0.95]$. This lets us fix $\bar p  = 0.93$ in our experiments.}
\resizebox{\linewidth}{!}{%
\begin{tabular}{@{}l cccccc@{}}
\toprule
& \makecell{Qwen-Math-1.5B\\(Numina-Math)} & \makecell{Qwen-Math-1.5B\\(OpenCode)} & \makecell{Qwen-Math-7B\\(Numina-Math)} & \makecell{Qwen-2.5-Instruct\\(Numina-Math)} & \makecell{Qwen-2.5-Instruct\\OpenCode)} & \makecell{Llama-3.1-8B \\(Opencode)} \\
\midrule
\makecell{Token\\Probability}
&  \makecell{93.6\\($\pm$ 0.1)} & \makecell{92.9\\($\pm$ 0.1)} & \makecell{92.7\\($\pm$ 0.1)} &  \makecell{93.0\\($\pm$ 0.1)} &  \makecell{91.8\\($\pm$ 0.1)} &  \makecell{90.5\\($\pm$ 0.1)}   \\
\bottomrule
\end{tabular}%
}
\label{table:pbar}
\end{table}

\begin{figure*}[h!]
    \centering
    \includegraphics[width=0.6\linewidth]{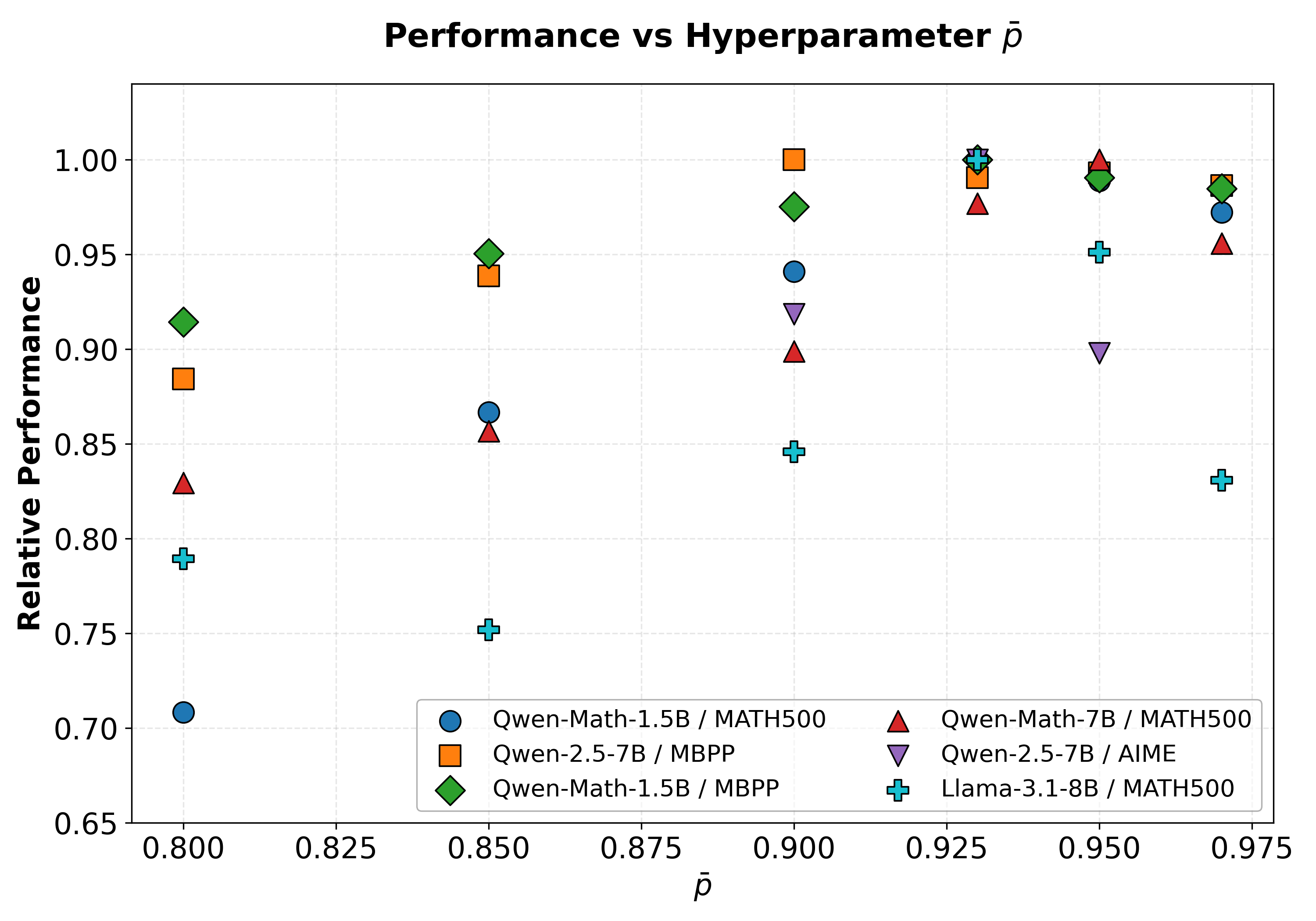}
    \vspace{-6pt}
    \caption{Relative performance of different models and datasets when sweeping hyper-parameter $\bar{p}$. We plot pass@1 for values of $\bar{p}$ and observe that $\bar{p} = 0.93 $ is a good choice across various settings.}
    \label{fig:pbar_realtive}
    \vspace{-2pt}
\end{figure*}

\subsection{Token accuracy comparisons}\label[appendixsection]{sec:tok_accuracy}
We report the training token accuracy for Qwen2.5-Math-1.5B and Qwen2.5-Math-7B during training on Numina-Math across all the methods. As shown in \Cref{fig:token_acc}, SFT has the highest accuracy since it assigns uniform weight to all the token samples. However, we show in \Cref{sec:performance} that this does transfer to better generalization at test-time, and SFT is outperformed by InfoSFT.

\begin{figure*}
    \centering
    \begin{subfigure}{0.4\linewidth}
        \includegraphics[width=\linewidth]{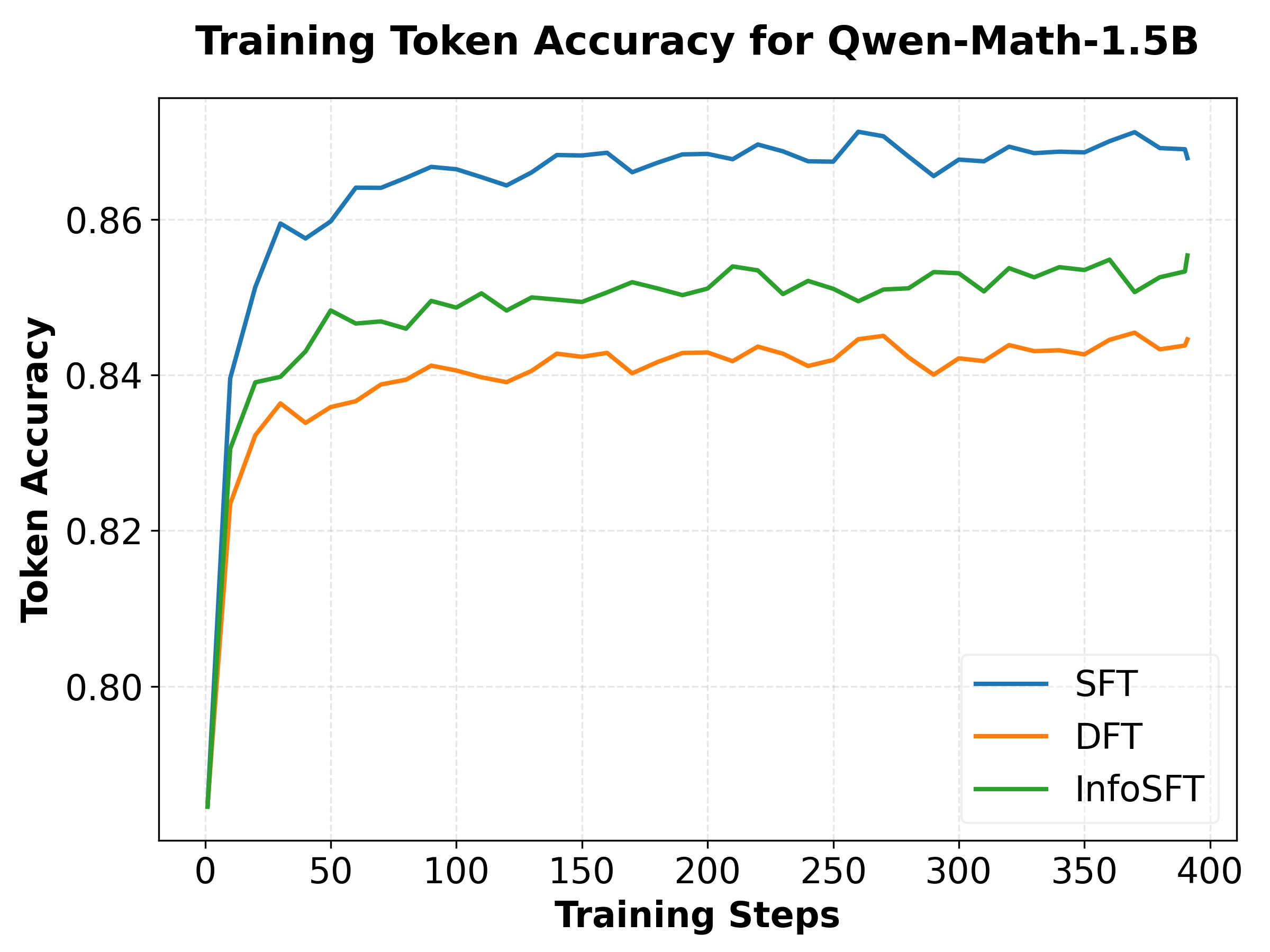}
    \end{subfigure}%
    \hspace{0.2in}
    \begin{subfigure}{0.4\linewidth}
        \includegraphics[width=\linewidth]{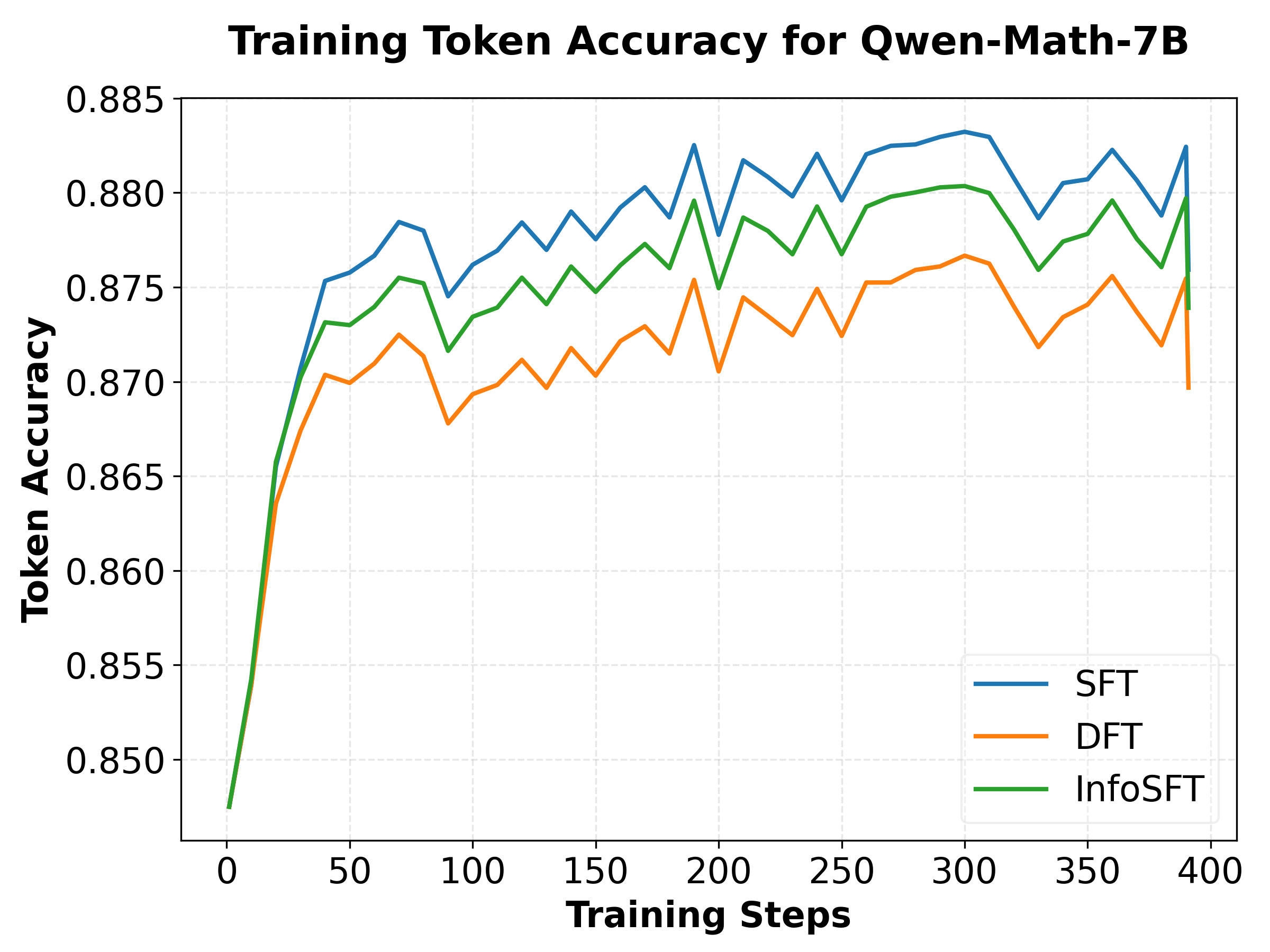}
    \end{subfigure}
     \caption{Token accuracy comparisons. SFT has higher training token accuracy, but it is outperformed by InfoSFT at test-time benchmarks.}
     \label{fig:token_acc}
\end{figure*}

\subsection{Detailed results of catastrophic forgetting experiments}\label[appendixsection]{sec:all_numbers}
We report the results of all the fine-tuning settings that were done in \Cref{sec:tradeoff} in \Cref{tab:science_results} and \Cref{tab:tool_results}.

\begin{table}[t]
\centering
\caption{Science Q\&A \citep{shenfeld2026self} fine-tuning results with SFT, DFT, and InfoSFT. All numbers are percentages. The numbers are used to make \Cref{fig:intro}, right. InfoSFT with $\mathrm{lr = 1e-5}$ and 2 epochs of training achieves the best result on the new task while forgetting less compared to DFT and SFT with the same hyper-parameters.}
\label{tab:science_results}
\resizebox{\linewidth}{!}{%
\begin{tabular}{lcccccc}
\toprule
Method & Setting & Science & IFEval & MATH500 & HumanEval & Prior Avg. \\
\midrule
SFT & $1\mathrm{e}{-6}$, 1 epoch & 51.5 & 73.9 & 72.2 & 82.9 & 76.4 \\
SFT & $1\mathrm{e}{-6}$, 2 epochs & 55.2 & 74.7 & 72.6 & 84.8 & 77.4 \\
SFT & $2\mathrm{e}{-6}$, 1 epoch & 55.0 & 71.6 & 71.0 & 82.9 & 75.2 \\
SFT & $2\mathrm{e}{-6}$, 2 epochs & 63.5 & 70.9 & 71.2 & 82.3 & 74.8 \\
SFT & $5\mathrm{e}{-6}$, 1 epoch & 63.5 & 66.7 & 67.6 & 82.3 & 72.2 \\
SFT & $5\mathrm{e}{-6}$, 2 epochs & 65.9 & 67.7 & 73.4 & 82.5 & 74.5 \\
SFT & $7\mathrm{e}{-6}$, 1 epoch & 64.5 & 65.8 & 69.0 & 81.1 & 72.0 \\
SFT & $7\mathrm{e}{-6}$, 2 epochs & 68.2 & 64.3 & 71.0 & 82.9 & 72.7 \\
SFT & $1\mathrm{e}{-5}$, 1 epoch & 67.1 & 62.1 & 68.2 & 81.1 & 70.5 \\
SFT & $1\mathrm{e}{-5}$, 2 epochs & 69.6 & 63.6 & 71.0 & 81.7 & 72.1 \\
\midrule
DFT & $1\mathrm{e}{-6}$, 1 epoch & 44.8 & 73.8 & 75.2 & 84.1 & 77.7 \\
DFT & $1\mathrm{e}{-6}$, 2 epochs & 53.5 & 72.8 & 76.2 & 82.9 & 77.3 \\
DFT & $2\mathrm{e}{-6}$, 1 epoch & 56.0 & 72.5 & 75.8 & 84.8 & 77.7 \\
DFT & $2\mathrm{e}{-6}$, 2 epochs & 53.7 & 71.2 & 75.4 & 82.9 & 76.5 \\
DFT & $5\mathrm{e}{-6}$, 1 epoch & 60.2 & 69.1 & 74.2 & 82.9 & 75.4 \\
DFT & $5\mathrm{e}{-6}$, 2 epochs & 62.3 & 69.1 & 74.6 & 81.1 & 74.9 \\
DFT & $7\mathrm{e}{-6}$, 1 epoch & 63.1 & 65.6 & 73.4 & 80.5 & 73.2 \\
DFT & $7\mathrm{e}{-6}$, 2 epochs & 62.7 & 69.1 & 73.8 & 82.9 & 75.3 \\
DFT & $1\mathrm{e}{-5}$, 1 epoch & 59.8 & 59.3 & 69.2 & 80.5 & 69.7 \\
DFT & $1\mathrm{e}{-5}$, 2 epochs & 66.5 & 62.3 & 71.0 & 81.7 & 71.7 \\
\midrule
InfoSFT & $1\mathrm{e}{-6}$, 1 epoch & 46.8 & 74.1 & 75.0 & 84.5 & 77.9 \\
InfoSFT & $1\mathrm{e}{-6}$, 2 epochs & 55.8 & 73.4 & 74.6 & 84.1 & 77.4 \\
InfoSFT & $2\mathrm{e}{-6}$, 1 epoch & 54.8 & 72.3 & 76.2 & 86.0 & 78.2 \\
InfoSFT & $2\mathrm{e}{-6}$, 2 epochs & 61.5 & 72.1 & 73.6 & 84.1 & 76.6 \\
InfoSFT & $5\mathrm{e}{-6}$, 1 epoch & 66.9 & 67.7 & 71.4 & 83.5 & 74.2 \\
InfoSFT & $5\mathrm{e}{-6}$, 2 epochs & 64.9 & 70.6 & 74.6 & 82.9 & 76.0 \\
InfoSFT & $7\mathrm{e}{-6}$, 1 epoch & 66.1 & 65.8 & 73.0 & 82.3 & 73.7 \\
InfoSFT & $7\mathrm{e}{-6}$, 2 epochs & 69.0 & 68.4 & 73.4 & 83.5 & 75.1 \\
InfoSFT & $1\mathrm{e}{-5}$, 1 epoch & 66.9 & 59.2 & 68.6 & 80.5 & 69.4 \\
InfoSFT & $1\mathrm{e}{-5}$, 2 epochs & 70.0 & 64.0 & 72.6 & 81.7 & 72.8 \\
\bottomrule
\end{tabular}%
}
\end{table}

\begin{table}[t]
\centering
\caption{Tool Use \citep{shenfeld2026self} fine-tuning results with SFT, DFT, and InfoSFT. All numbers are percentages. Even though SFT achieves the best performance on the new task ($\mathrm{lr = 7e-6}$), it also forgets more compared to InfoSFT with $\mathrm{lr = 5e-6, 2epochs}$.}
\label{tab:tool_results}
\resizebox{\linewidth}{!}{%
\begin{tabular}{lcccccc}
\toprule
Method & Setting & Tool Use & IFEval & MATH500 & HumanEval & Prior Avg. \\
\midrule
SFT & $1\mathrm{e}{-6}$, 1 epoch & 61.9 & 73.8 & 71.8 & 82.3 & 76.0 \\
SFT & $1\mathrm{e}{-6}$, 2 epochs & 62.9 & 69.9 & 74.2 & 81.1 & 75.1 \\
SFT & $2\mathrm{e}{-6}$, 1 epoch & 63.9 & 67.7 & 73.2 & 81.1 & 74.0 \\
SFT & $2\mathrm{e}{-6}$, 2 epochs & 66.0 & 68.2 & 75.2 & 82.3 & 75.2 \\
SFT & $5\mathrm{e}{-6}$, 1 epoch & 67.0 & 62.7 & 68.2 & 76.8 & 69.2 \\
SFT & $5\mathrm{e}{-6}$, 2 epochs & 67.0 & 62.5 & 68.8 & 78.0 & 69.8 \\
SFT & $7\mathrm{e}{-6}$, 1 epoch & 68.0 & 58.4 & 55.2 & 77.4 & 63.7 \\
SFT & $7\mathrm{e}{-6}$, 2 epochs & 69.2 & 60.4 & 64.2 & 76.8 & 67.2 \\
\midrule
DFT & $1\mathrm{e}{-6}$, 1 epoch & 64.0 & 72.6 & 75.4 & 83.5 & 77.2 \\
DFT & $1\mathrm{e}{-6}$, 2 epochs & 66.0 & 71.9 & 75.0 & 82.9 & 76.6 \\
DFT & $2\mathrm{e}{-6}$, 1 epoch & 65.0 & 70.2 & 75.2 & 82.3 & 75.9 \\
DFT & $2\mathrm{e}{-6}$, 2 epochs & 62.9 & 69.1 & 73.8 & 81.7 & 74.9 \\
DFT & $5\mathrm{e}{-6}$, 1 epoch & 62.9 & 67.7 & 73.8 & 77.4 & 73.0 \\
DFT & $5\mathrm{e}{-6}$, 2 epochs & 62.9 & 64.1 & 75.2 & 75.0 & 71.5 \\
DFT & $7\mathrm{e}{-6}$, 1 epoch & 65.0 & 60.3 & 73.4 & 73.2 & 69.0 \\
DFT & $7\mathrm{e}{-6}$, 2 epochs & 67.0 & 60.6 & 70.2 & 73.2 & 68.0 \\
\midrule
InfoSFT & $1\mathrm{e}{-6}$, 1 epoch & 65.0 & 72.3 & 77.4 & 82.3 & 77.3 \\
InfoSFT & $1\mathrm{e}{-6}$, 2 epochs & 65.9 & 74.1 & 74.8 & 82.3 & 77.1 \\
InfoSFT & $2\mathrm{e}{-6}$, 1 epoch & 62.9 & 74.6 & 73.4 & 81.7 & 76.6 \\
InfoSFT & $2\mathrm{e}{-6}$, 2 epochs & 67.0 & 69.7 & 74.8 & 80.5 & 75.0 \\
InfoSFT & $5\mathrm{e}{-6}$, 1 epoch & 60.8 & 74.5 & 74.8 & 81.8 & 77.0 \\
InfoSFT & $5\mathrm{e}{-6}$, 2 epochs & 68.0 & 66.4 & 67.0 & 78.7 & 70.7 \\
InfoSFT & $7\mathrm{e}{-6}$, 1 epoch & 63.9 & 62.7 & 67.2 & 76.2 & 68.7 \\
InfoSFT & $7\mathrm{e}{-6}$, 2 epochs & 65.0 & 63.2 & 70.2 & 76.8 & 70.1 \\
\bottomrule
\end{tabular}%
}
\end{table}

\section{InfoSFT Controls the Training Entropy}\label[appendixsection]{sec:entropy}
In this section, we study the behavior of the loss function defined by InfoSFT. For simplicity, we only consider that case that $q \le \bar p $, and neglect the clipping term in \Cref{eq:InfoSFT_grad}. This will result in the following: 
\begin{align*}
\nabla_\theta J_{\text{InfoSFT}}
& \approx 
\E_{(x,y^*)\sim \calD}\bigg[\frac{1}{L}\sum_{t=1}^{L} \pi_\theta(y^*_t|x) \big( \mathrm{logit}(\bar p) - \mathrm{logit}(\pi_\theta(y^*_t|x)) \big) \nabla_\theta \log \pi_\theta(y^*_t \mid x)\bigg] \\
& = \E_{(x,y^*)\sim \calD}\bigg[\frac{1}{L}\sum_{t=1}^{L} \big( \mathrm{logit}(\bar p) - \mathrm{logit}(\pi_\theta(y^*_t|x)) \big) \nabla_\theta  \pi_\theta(y^*_t \mid x)\bigg]
\end{align*}

\noindent Using that $\int \log(x) dx  = x (\log(x) -1) $, the corresponding objective is:
\begin{align}
    J_{\mathrm{InfoSFT}}(\theta)
& =
\E_{(x,y^*)\sim \calD}\bigg[\frac{1}{L}\sum_{t=1}^{L} \big( \mathrm{logit}(\bar p) \pi_\theta(y^*_t \mid x) - \pi_\theta(y^*_t \mid x) \log(\pi_\theta(y^*_t|x)) \nonumber\\
& \hspace{2.355in} - (1-\pi_\theta(y^*_t \mid x))\log(1- \pi_\theta(y^*_t|x)) \big) \bigg] \nonumber\\
& =  \mathrm{logit}(\bar p) \cdot J_{\mathrm{DFT}} - \underbrace
{\E \bigg[ - \pi_\theta(y^*_t \mid x) \log(\pi_\theta(y^*_t|x)) - (1-\pi_\theta(y^*_t \mid x))\log(1- \pi_\theta(y^*_t|x)) \big) \bigg]}_{H_b(y^*|x)}
\end{align}

\noindent Thus InfoSFT is approximately equivalent to adding the binary entropy $H_b(y^*|x)$ term to DFT with an appropriate coefficient. We note that this term is different from the empirical entropy that is: $$\E\big[\sum_y - \pi_\theta(y|[x, y^*_{<t}]) \log(\pi_\theta(y|[x, y^*_{<t}])\big]$$ InfoSFT only accounts for the tokens in the expert response, and adds a tight upper-bound (when the rest of the tokens are uniform) for the rest of the tokens. Nevertheless, since supervised training only focuses on the expert tokens, the impact is roughly the same as shown in \Cref{fig:pass@k_entropy} (right).


\section{Proofs}
\begin{proof}[\textbf{Proof of Lemma \ref{lem:oracle}}] \textbf{(a)} Write $\pi (y) = \pi_{0}(y)/Z$ for $y \neq y^*$ and $\pi (y^*) = \pi_{0}(y^*)e^u/Z$ with $u = \Omega/\beta$. Then:
\begin{align*} \KL(p^* \| \pi ) &= \sum_y p^*(y)\log\frac{p^*(y)}{\pi (y)} = \sum_{y\neq y^*} p^*(y)\log\frac{p^*(y)\,Z}{\pi_{0}(y)} + p^*(y^*)\log\frac{p^*(y^*)\,Z}{\pi_{0}(y^*)\,e^u} \\ 
&= 
 \sum_{y\neq y^*} p^*(y)\big[\log\frac{p^*(y)}{\pi_{0}(y)} + \log Z \big] + p^*(y^*)\log\big[\frac{p^*(y^*)}{\pi_{0}(y^*)} - u + \log Z \big] \\
&= \KL(p^* \| \pi_{0}) + \log Z - p\,u
\end{align*}

\noindent In which $p = p^*(y^*)$ according to our notation. Subtracting $\KL(p^* \| \pi_{0})$ from both sides obtains the final expression.  

\vspace{4pt}
\noindent \textbf{(b)} 
Remember that $Z = q e^u +1 -q$ and $q = \pi_{0}(y^*)$. Then,
differentiating with respect to~$u$ yields:
\[ \frac{\partial\,\Delta\KL}{\partial u} = \frac{q\,e^u}{q e^u +1 -q} - p = q\frac{e^u}{Z} - p =  \pi (y^*\!\mid\! x) - p \]
Setting to zero yields $\pi (y^* \mid x) = p^*$, which upon solving for~$u$ gives

$$\frac{q\,e^{u^*}}{q\,e^{u^*}+1-q} = p \quad\Longrightarrow\quad e^{u^*} = \frac{p(1-q)}
{q(1-p)} \quad\Longrightarrow\quad u^* = \mathrm{logit}(p)-\mathrm{logit}(q).$$

\noindent Taking the second derivative:

$$\frac{\partial^2\Delta\KL}{\partial u^2} = \frac{qe^u(qe^u+1-q)-qe^u(qe^u)}{(qe^u+1-q)^2} = \frac{qe^u}{(qe^u+1-q)}.\frac{qe^u}{(1-q)}$$

\noindent where the first fraction in RHS is $\pi (y^*) = \pi_{0}(y^*)e^u/Z$. Thus: $\partial^2\Delta\KL/\partial u^2 = \pi (y^*)(1-\pi (y^*))>0$ confirms strict convexity.


\end{proof}

\begin{proof}[\textbf{Proof of Lemma \ref{lem:expected}}]
\textbf{(a)} Starting from \Cref{eq:expected-kl}, we plug-in $u_C = C - \mathrm{logit}(q)$: 

$$\Delta KL(u_C) = \E_q[\log(\frac{e^C}{1-q} + 1 - q)] - \bar p(q) C + \bar p (q) \mathrm{logit}(q)] $$

\noindent Taking the gradient with respect to $C$ gives: 

\begin{equation}\label{eq:grad_C}
    \frac{\partial \Delta KL(u_C)}{\partial C} = \E_q[\frac{e^C}{e^C + (1-q)^2}] - \bar p(q)]
\end{equation}

\noindent Setting the derivative to zero gives:

\begin{equation*}
\E_q\!\Big[\frac{e^{C^*}}{e^{C^*}+(1-q)^2}\Big] = \bar p
\end{equation*}
If \(q \le d \ll 1\), then:

$$\E_q\!\Big[\frac{e^{C^*}}{e^{C^*}+1 }(1 + \frac{2q}{e^{C^*}+1})\Big] = \frac{e^{C^*}}{e^{C^*}+1 }(1 + \frac{2\E[q]}{e^{C^*}+1}) = \bar p $$ And therefore: 

$$C^* = \mathrm{logit}(\bar p) +\calO(d)$$

Additionally, taking the second derivative from \Cref{eq:grad_C}:

$$\frac{\partial^2 \Delta KL(u_C)}{\partial^2 C} = \E_q[\frac{e^C(e^C + (1-q)^2) - e^C.e^C}{(e^C + (1-q)^2)^2}] = \E_q[\frac{e^C (1-q)^2}{(e^C + (1-q)^2)^2}] > 0 $$

\noindent Thus, $\Delta KL(u_C)$ is strictly convex, and $C^*$ is its unique minimizer.

\vspace{4pt}
\noindent \textbf{(b)} From the definition of $\Delta KL(u) = \log(qe^u + 1 - q) - pu$, we substitute the Oracle weight $u^* = \log \frac{p}{1-p} - \log \frac{q}{1-q}$:
\begin{align}\label{eq:KL_oracle}
\Delta KL(u^*) &= \log\left(q \frac{p}{1-p} \frac{1-q}{q} + 1 - q\right) - p \left( \log \frac{p}{1-p} - \log \frac{q}{1-q} \right) \nonumber \\
&= \log\left(\frac{1-q}{1-p}\right) - p \left( \log \frac{p}{1-p} - \log q + \log(1-q) \right) \nonumber \\
&= p \log q + (1-p) \log(1-q) - (1-p) \log(1-p) - p \log p
\end{align}

\noindent We can reuse the expansion for KL in~\eqref{eq:KL_oracle} for the calibrated InfoSFT rule with $\bar{p}$ replacing $p$ inside $u^*$:

\begin{align}\label{eq:expected_InfoSFT}
    \E[\Delta\KL_{\mathrm{info}}] = & 
    \E \Big[ \log\left(\frac{1-q}{1- \bar p}\right) - p \left( \log \frac{\bar p}{1- \bar p} - \log \frac{q}{1-q} \right) \Big] \nonumber \\
= & \E [p \log q] + \E[(1-p) \log(1-q)] - (1- \bar p) \log(1- \bar p) - \bar p \log \bar p \nonumber \\
=& \E [p \log q] + \E[(1-p) \log(1-q)] + H_b(\bar p) 
\end{align}

\noindent Likewise, for the oracle:
$$\E[\Delta\KL^*] =  \E [p \log q] + \E[(1-p) \log(1-q)] + \E [ H_b(p)] $$
Subtracting yields:
\begin{equation}\label{eq:expected_oracle}
    \E[\Delta\KL_{\mathrm{info}}]-\E[\Delta\KL^*] = H_b(\bar{p})-\E[H_b(p)]
\end{equation}

\noindent Furthermore, since \(H_b\) is concave, by Jensen's inequality we have:
\begin{equation}\label{eq:jensen}
    H_b(\bar p)=H_b\!\bigl(\E[p]\bigr)\ge \E[H_b(p)]
\end{equation}
which proves the non-negativity. 
For the ratio, we begin with showing that $\E[\Delta\KL^*]$ is negative:

\begin{align}\label{eq:expected_oracle_negative}
    \E[\Delta\KL^*] &=  1 + \E [p \log q + \underbrace{(1-p) \log(1-q)}_{\leq 0} - p \log p - (1 - p) \log(1-p)] \nonumber \\ 
    &  \leq \E[ p\log q - p \log p + p ] \nonumber\\
    &  \leq \E[ p\log d] -  \E [p \log p] + E[p] \nonumber\\
    &  \leq \E[ p\log d] -  \E [p] \log(\E[p]) + E[p] \nonumber\\
    & = \bar p \log d - \bar p \log \bar p  + \bar p =  \log (\frac{d}{\bar p/ e})
    < 0
\end{align}

\noindent In the second line, we use that $-(1-x) \log(1-x) \leq x$ where $x \leq 1$, and in the third line we use that $ q < d$. In the fourth line, we use the Jensen inequality for the concave function: $-x\log(x)$. The last inequality is implied by $d \leq \bar p / e^2$. Now, we can simplify the ratio: 

\begin{align*}
    \frac{\E[\Delta KL(u_{\text{info}})]}{\E[\Delta KL(u^*)]} &= 1 + \frac{H_b(\bar p)}{{\E[\Delta KL(u^*)]}} = 
1 - \frac{H_b(\bar p) - \E[H_b(p)]}{|\E [p \log q + (1-p) \log(1-q) + H_b(p)]|} \\
& \geq 1 - \frac{H_b(\bar p)}{|\E [p \log q + H_b(p)]|} \\
& \geq 1- \frac{-{\bar p} \log(\bar p) + \bar p}{\E [p \log(1/d) -  \bar p \log(\bar p/e)]|} \\
& = 1- \calO (\frac{-{\bar p} \log(\bar p) + \bar p}{\E [p] \log (1/d) }) = 1 - \calO(\frac{\log(1/\bar p ) +1}{\log (1/d) }) = 1 - \calO(\frac{\max \{\log(1/\bar p), 1\}}{\log (1/d) })
\end{align*}

\vspace{4pt}
\noindent \textbf{(c)} We need to show that $\E[\Delta\KL_{\mathrm{info}}] < \E[\Delta\KL_{\mathrm{DFT}}]$ so that InfoSFT gets closer to the underlying distribution. We can start from \Cref{eq:expected_InfoSFT}:
\begin{align*}
    \E[\Delta\KL_{\mathrm{info}}] \leq \E [p \log q] + H_b(\bar p) \leq \bar p \log d + H_b(\bar p)
\end{align*}
For DFT, recall that $\Omega(q) = 1$ as previously discussed; note that this choice holds up to a constant factor, as the key consideration is the relative weighting of samples. Scaling $\Omega$ by a constant merely adjusts the tradeoff between expected reward and KL divergence in formulation~\eqref{eq:proximal}---equivalently, controlling the permissible deviation from the base policy $\pi_0$. In practice, this is managed through policy gradient learning rate and stopping time. Thus, without loss of generality, we may take $u_{\mathrm{DFT}}=1$. Substituting \(u_{\mathrm{DFT}}=1\) into Lemma~\ref{lem:oracle} gives
\[
\E[\Delta\KL_{\mathrm{DFT}}]
=
\E\!\left[\log\!\bigl(1+(e-1)q\bigr)\right]-\bar p,
\]

\noindent Therefore a sufficient condition for $\E[\Delta\KL_{\mathrm{info}}] < \E[\Delta\KL_{\mathrm{DFT}}]$ is:
\[
\bar p\,|\log d|-H_b(\bar p)>\bar p,
\]
Finally, using the inequality $-(1-x) \log(1-x) \leq x$ again that was used in the derivation of \Cref{eq:expected_oracle_negative}:
\[
H_b(\bar p)\le -\bar p\log\bar p+\bar p=\bar p(1-\log\bar p),
\]
it is enough that
\[
|\log d|\ge 2-\log\bar p = -\log(\frac{\bar p}{e^2})
\]
which shows that $d \leq \frac{\bar p}{e^2}$ is sufficient for the dominance of InfoSFT.

For SFT, we remind that $u(q) = 1/q$. Define:
\[
\psi(t):=\log\!\left(1+t(e^{1/t}-1)\right).
\]
For \(0<t<1\):
\begin{equation}\label{eq:psi}
    \psi(t)\ge \log(te^{1/t})=\frac1t+\log t
\end{equation}
and \(\psi\) is decreasing on \((0,1)\), since
\[
\psi'(t)
=
\frac{e^{1/t}(1-1/t)-1}{1+t(e^{1/t}-1)}<0
\]
Therefore:
\begin{align*}
\E[\Delta\KL_{\mathrm{SFT}}]
&=\E\!\left[\psi(q)-\frac{p}{q}\right] \\
&=\E[p\log q]
  +\E\!\left[\psi(q)-p\left(\frac1q+\log q\right)\right] \\
&=\E[p\log q]
  +\E\!\left[(1-p)\psi(q)
  +p\left(\psi(q)-\frac1q-\log q\right)\right] \\
&\ge \E[p\log q]+\E[(1-p)\psi(q)] \\
&\ge \E[p\log q]+(1-\bar p)\psi(d).
\end{align*}
In the second line, we add and subtract $p\log q$. In the fourth line, we first use that $\psi(q) \geq 1/q  + \log q$ per \Cref{eq:psi}. In the fifth line, $\psi(q) \geq \psi(d)$ since $q \leq d$ and $\psi$ is a decreasing function. Combining this with: 
\[
\E[\Delta\KL_{\mathrm{info}}]
\le \E[p\log q]+H_b(\bar p)
\]
it is enough to show that: 
\[
(1-\bar p)\psi(d)\ge H_b(\bar p)
\]
Since \(q<d<\bar p/e^2\) and \(\psi\) is decreasing:
\[
\psi(d)\ge \psi(\bar p/e^2)
=
\log\!\left(
1+\frac{\bar p}{e^2}\left(e^{e^2/\bar p}-1\right)
\right)
\]
Hence it suffices that:
\[
G(\bar p)\ge 0,\
\]
where
\[
G(x)
:=
(1-x)\log\!\left(
1+\frac{x}{e^2}\left(e^{e^2/x}-1\right)
\right)
+x\log x+(1-x)\log(1-x)
\]
It can be numerically verified that $G(x) > 0 $ for $x \leq 0.988$. Then:
\[
\E[\Delta\KL_{\mathrm{SFT}}]
\ge
\E[p\log q]+H_b(\bar p)
\ge
\E[\Delta\KL_{\mathrm{info}}]
\]
which proves the comparison.

\end{proof}

\section{Model responses of InfoSFT and SFT}\label[appendixsection]{sec:examples}

We include representative model responses to illustrate the response format learned from OpenR1-Math. Using InfoSFT only for training does not teach the "<think>" token to the model since it's probability is very low under the base model. As the response below shows, the answer is much shorter than the model trained with SFT. Nevertheless, InfoSFT still improves the model's accuracy on AMC. The response of the model after 1 epoch of SFT and 1 epoch InfoSFT is quite similar to the one trained with SFT only.

\begin{responsebox}{AIME24 (Question 1) - Model trained with 2 epochs of InfoSFT}

\small

\textbf{Problem.}
Every morning Aya goes for a $9$-kilometer-long walk and stops at a coffee shop afterwards. When she walks at a constant speed of $s$ kilometers per hour, the walk takes her 4 hours, including $t$ minutes spent in the coffee shop. When she walks $s+2$ kilometers per hour, the walk takes her 2 hours and 24 minutes, including $t$ minutes spent in the coffee shop. Suppose Aya walks at $s+\frac{1}{2}$ kilometers per hour. Find the number of minutes the walk takes her, including the $t$ minutes spent in the coffee shop.

\vspace{0.4em}

\textbf{Model response.}

\begin{lstlisting}[style=modeloutput]
Sure, let's tackle this problem step by step. So, Aya has a 9-kilometer walk, and she stops at a coffee shop after that. The problem gives us two different walking speeds and the total time including the coffee shop time. We need to find the time when she walks at a speed of \( s + \frac{1}{2} \) km/h.

First, let's parse the information given:

1. When walking at speed \( s \) km/h, the total time including coffee shop time is 4 hours.
2. When walking at speed \( s + 2 \) km/h, the total time including coffee shop time is 2 hours and 24 minutes.
3. We need to find the total time when she walks at \( s + \frac{1}{2} \) km/h.

Let me note that 2 hours and 24 minutes is equal to 2.4 hours. So, the total time in hours for the second case is 2.4 hours.

Let's denote:

- \( s \): original speed in km/h
- \( t \): time spent in the coffee shop in hours
- \( d = 9 \) km: distance of the walk

First, when she walks at speed \( s \):

Total time = walking time + coffee shop time = \( \frac{d}{s} + t \)

Given that this total time is 4 hours, so:

\( \frac{9}{s} + t = 4 \)  ...(1)

Second, when she walks at speed \( s + 2 \):

Total time = walking time + coffee shop time = \( \frac{9}{s + 2} + t \)

Given that this total time is 2.4 hours, so:

\( \frac{9}{s + 2} + t = 2.4 \)  ...(2)

We need to solve these two equations to find \( s \) and \( t \), then use them to find the total time when she walks at \( s + \frac{1}{2} \) km/h.

Let's subtract equation (2) from equation (1) to eliminate \( t \):

\( \frac{9}{s} + t - \left( \frac{9}{s + 2} + t \right) = 4 - 2.4 \)

Simplifying:

\( \frac{9}{s} - \frac{9}{s + 2} = 1.6 \)

Let me compute the left-hand side (LHS):

\( 9 \left( \frac{1}{s} - \frac{1}{s + 2} \right) = 9 \left( \frac{(s + 2) - s}{s(s + 2)} \right) = 9 \left( \frac{2}{s(s + 2)} \right) = \frac{18}{s(s + 2)} \)

So:

\( \frac{18}{s(s + 2)} = 1.6 \)

Multiply both sides by \( s(s + 2) \):

\( 18 = 1.6 s(s + 2) \)

Divide both sides by 1.6:

\( \frac{18}{1.6} = s(s + 2) \)

Calculate \( 18 / 1.6 \):

1.6 * 10 = 16, 18 -16 = 2, so 18 /1.6 = 11.25

So:

\( 11.25 = s^2 + 2s \)

Multiply both sides by 4 to eliminate decimals:

\( 45 = 4s^2 + 8s \)

Bring all terms to one side:

\( 4s^2 + 8s - 45 = 0 \)

Now, solve this quadratic equation for \( s \):

Using quadratic formula:

\( s = \frac{ -8 \pm \sqrt{8^2 - 4 \times 4 \times (-45) } }{ 2 \times 4 } \)

Compute discriminant:

\( 64 - 4 \times 4 \times (-45) \)

= 64 + 720 = 784

So,

\( s = \frac{ -8 \pm \sqrt{784} }{ 8 } \)

sqrt(784) = 28

Thus,

\( s = \frac{ -8 \pm 28 }{ 8 } \)

We have two solutions:

1. \( s = \frac{20}{8} = 2.5 \) km/h
2. \( s = \frac{ -36 }{8 } = -4.5 \) km/h

Since speed can't be negative, we discard the negative solution. So, \( s = 2.5 \) km/h.

Now, substitute \( s = 2.5 \) into equation (1) to find \( t \):

\( \frac{9}{2.5} + t = 4 \)

Calculate \( 9 / 2.5 \):

2.5 * 3 = 7.5, 9 -7.5 =1.5, so 9 /2.5 = 3.6

Thus,

3.6 + t = 4

Therefore, \( t = 4 - 3.6 = 0.4 \) hours.

Convert 0.4 hours to minutes: 0.4 *60 =24 minutes.

Check with the second equation to verify:

\( \frac{9}{2.5 + 2} + t = \frac{9}{4.5} + 0.4 = 2 + 0.4 = 2.4 \) hours, which matches the given 2 hours 24 minutes. So that checks out.

Now, we need to find the total time when she walks at \( s + \frac{1}{2} = 2.5 + 0.5 = 3.0 \) km/h.

Total time = walking time + coffee shop time = \( \frac{9}{3.0} + 0.4 \)

Calculate \( 9 /3.0 \):

3.0 *3 =9, so 9 /3.0 =3

Thus, total time = 3 + 0.4 =3.4 hours.

Convert 3.4 hours to minutes: 3 *60 =180, 0.4*60=24, so total 180 +24=204 minutes.

Therefore, the answer is 204 minutes.

Let me verify once again:

At 3.0 km/h, time taken for 9 km is 9 /3 =3 hours. Plus 24 minutes coffee shop time: 3 hours 24 minutes =204 minutes. Correct.

Hence, the number of minutes is \boxed{204}.
\end{lstlisting}

\end{responsebox}

\begin{responsebox}{AIME24 (Question 1) - Model trained with 2 epochs of SFT}

\small

\textbf{Problem.}
Every morning Aya goes for a $9$-kilometer-long walk and stops at a coffee shop afterwards. When she walks at a constant speed of $s$ kilometers per hour, the walk takes her 4 hours, including $t$ minutes spent in the coffee shop. When she walks $s+2$ kilometers per hour, the walk takes her 2 hours and 24 minutes, including $t$ minutes spent in the coffee shop. Suppose Aya walks at $s+\frac{1}{2}$ kilometers per hour. Find the number of minutes the walk takes her, including the $t$ minutes spent in the coffee shop.

\vspace{0.4em}

\textbf{Model response.}

\begin{lstlisting}[style=modeloutput]
<think>
Okay, so I need to find the value of the expression 3 plus 1 over (3 plus 1 over (3 plus 1 over 3)), and then write it as a reduced fraction m/n. Then find m + n. Hmm, let's break this down step by step. 

First, let me parse the expression correctly. It's 3 + 1 divided by (3 + 1 divided by (3 + 1/3)). So starting from the innermost part, which is 3 + 1/3. Then that result is in the denominator of another fraction, which is 1 divided by it, and then that whole thing is added to 3. So maybe I should work from the inside out, simplifying each part step by step.

Let me write this out:

First, compute the innermost fraction: 3 + 1/3. That's straightforward. To add 3 and 1/3, I can write 3 as 9/3, so 9/3 + 1/3 = 10/3. Okay, so that part simplifies to 10/3.

Now, moving up one level, the next expression is 3 + 1 divided by (that result). So substituting, it's 3 + 1/(10/3). Let me compute 1 divided by 10/3. Dividing by a fraction is the same as multiplying by its reciprocal, so 1 * (3/10) = 3/10. So now, the expression becomes 3 + 3/10. 

Again, adding 3 and 3/10. Convert 3 to 30/10, so 30/10 + 3/10 = 33/10. So now, the entire original expression simplifies to 33/10. Wait, but let me check if I did that correctly. Let me verify each step again.

First step: 3 + 1/3. 3 is 9/3, 9/3 + 1/3 = 10/3. Correct. Then, 1 divided by (10/3) is 3/10. Then 3 + 3/10. 3 is 30/10, so 30/10 + 3/10 = 33/10. Yes, that seems right. So the entire expression simplifies to 33/10. 

But wait, the problem says "the Irreducible fraction value of [the expression]". So 33 and 10 are both primes? Let's check. 33 is 3*11, and 10 is 2*5. So 33/10 cannot be simplified further because there are no common factors between 33 and 10. Therefore, m is 33 and n is 10. Then m + n is 43. 

But wait, let me make sure I didn't make a mistake in the calculation. Let me go through it again step by step. 

Starting with the innermost part: 3 + 1/3. Yes, that's 10/3. Then, moving up, it's 3 plus 1 divided by (10/3). So 1 divided by (10/3) is 3/10. Then, 3 + 3/10. 3 is 30/10, so 30/10 + 3/10 = 33/10. So that's correct. So 33/10 is already in lowest terms. Therefore, m + n is 33 + 10 = 43. So the answer is 43. 

But let me check once again. Maybe I misread the expression. Let me write it out in order of operations. The expression is 3 + 1/(3 + 1/(3 + 1/3)). So starting from the innermost parentheses: 3 + 1/3 is 10/3. Then, 3 + 1/(10/3) is 3 + 3/10, which is 33/10. Then, adding 3 to that? Wait, hold on! Wait, the entire expression is 3 + 1 divided by (3 + 1 divided by (3 + 1/3)). So it's 3 plus 1 over (the previous result). So 33/10 is the result of that. Therefore, the entire expression is 33/10. So that's correct. 

Alternatively, maybe there's a different way to parse the expression? Let me check. If the expression were 3 + 1 divided by [3 + 1 divided by (3 + 1/3)], then yes, that's how I interpreted it. So 3 + 1 is the numerator, and the denominator is 3 + 1/(3 + 1/3). So yes, the calculation is correct. 

Alternatively, if it were 3 + [1/(3 + 1/(3 + 1/3))], then it would be the same. So regardless of the parentheses, the calculation is straightforward. 

Alternatively, maybe someone might misinterpret the order of operations, but given the way it's written, with all the fractions stacked, it's clear that the innermost part is 3 + 1/3, then that result is in the denominator. So I think my initial steps are correct. 

Alternatively, perhaps I can write the entire expression step by step using fractions:

Let me denote the innermost part as x = 3 + 1/3 = 10/3.

Then the next layer is y = 3 + 1/x = 3 + 3/10 = 33/10.

Then the entire expression is z = 3 + 1/y = 3 + 10/33 = 99/33 + 10/33 = 109/33. Wait, hold on! Wait, that contradicts the previous result. Wait, now I'm confused. Wait, which is it?

Wait, hold on. Let me do this again. Wait, perhaps I made a mistake in the last step. Let me redo the entire calculation.

Original expression: 3 + 1/(3 + 1/(3 + 1/3))

First, compute the innermost fraction: 3 + 1/3 = 10/3.

Then, the next layer up is 3 + 1/(10/3). So 1 divided by (10/3) is 3/10. Then 3 + 3/10. So 3 is 30/10, so 30/10 + 3/10 = 33/10. So that's the second layer. Then, the entire expression is 3 + 1/(33/10). Wait, no, hold on. Wait, let me parse the original expression again.

Wait, original expression: 3 + 1/(3 + 1/(3 + 1/3)). So after computing the innermost part, which is 3 + 1/3 = 10/3, then the next part is 3 + 1/(10/3). So that's 3 + (3/10) = 33/10. Then, the entire expression is 3 + 1/(33/10). Wait, no, hold on. Wait, no. Wait, the original expression is 3 + 1 divided by (3 + 1 divided by (3 + 1/3)). So if after computing the innermost part, we have 10/3, then the next part is 3 + 1/(10/3). Then that result is added to 3. So yes, that's 33/10. Then, the entire expression is 3 + 1/(33/10). Wait, no, hold on. Wait, no. Wait, the original expression is 3 + 1 divided by [3 + 1 divided by (3 + 1/3)]. So the [ ] part is 33/10. So the entire expression becomes 3 + 1/(33/10). Wait, no. Wait, the original expression is 3 + [1 divided by (3 + 1/(3 + 1/3))]. So "3 + [something]". So if the "something" is 33/10, then the entire expression is 3 + 33/10. Which is 33/10 + 30/10 = 63/10? Wait, no, wait, that can't be. Wait, hold on, this is where I messed up earlier. 

Wait, no. Wait, 3 is 30/10, and 33/10 is the result of the previous step. So adding 30/10 + 33/10 = 63/10. Wait, but that contradicts the previous step. Wait, so where did I go wrong?

Wait, let's go back. The original expression is 3 + 1/(3 + 1/(3 + 1/3)). Let's compute step by step:

Start with the innermost part: 3 + 1/3 = 10/3.

Then, the next layer is 3 + 1/(10/3). Let's compute that. 1 divided by (10/3) is 3/10. Then, 3 + 3/10 = 33/10. So now, the expression is 3 + 1/(33/10). Wait, no. Wait, the original expression is 3 + 1 divided by (the previous result). Wait, no, the original expression is 3 + [1 divided by (3 + 1/(3 + 1/3))]. So after computing 3 + 1/(10/3) = 33/10, then the entire expression is 3 + 1/(33/10). But that would be 3 + 10/33. Which is 99/33 + 10/33 = 109/33. But that contradicts the previous result. So which is correct?

Wait, now I'm confused. So let's parse the original expression again:

3 + 1/(3 + 1/(3 + 1/3))

So starting from the innermost parentheses: 3 + 1/3 = 10/3.

Then, moving up, the next layer is 3 + 1/(10/3). Let's compute that. 1 divided by (10/3) is 3/10. Then, 3 + 3/10 = 33/10. So the expression is now 3 + 1/(33/10). Wait, no, the original expression is 3 + [1 divided by (the previous result)]. Wait, no. Wait, the expression is 3 + 1/(3 + 1/(3 + 1/3)). So after computing the innermost part as 10/3, then the next part is 3 + 1/(10/3) = 33/10. Then, the entire expression is 3 + 1/(33/10). Wait, that can't be. Wait, no. Wait, the original expression is 3 + [1/(3 + 1/(3 + 1/3))]. So the "1/" part is only applied to the innermost denominator, right? So first, compute 1/(3 + 1/3) = 1/(10/3) = 3/10. Then, compute 3 + 3/10 = 33/10. Then, the entire expression is 3 + 1/(33/10). Wait, no, that would be 3 + 10/33. Which is 99/33 + 10/33 = 109/33. But that contradicts the previous result. So which is correct?

Wait, this is confusing. Let me re-express the original problem:

Original expression: 3 + 1/(3 + 1/(3 + 1/3))

Let me compute this step by step:

Let's denote the innermost part as follows:

Start with the innermost parentheses: 3 + 1/3 = 10/3.

Next layer: 3 + 1/(10/3) = 3 + 3/10. Let's compute this. 3 is 30/10, so 30/10 + 3/10 = 33/10. 

Then, the entire expression is 3 + 1/(33/10). Wait, no. Wait, the original expression is 3 + [1/(3 + 1/(3 + 1/3))]. So if the innermost part is 10/3, then the next layer is 3 + 1/(10/3) = 33/10. Then, the entire expression is 3 + 1/(33/10). Wait, no. Wait, no. Wait, the original expression is 3 plus 1 divided by [3 + 1 divided by (3 + 1/3)]. So after computing the innermost part as 10/3, then the next layer is 3 + 1/(10/3) = 33/10. Then, the entire expression is 3 + 1/(33/10). Wait, but that would be 3 + 10/33. Wait, that can't be. Wait, no. Wait, the original expression is 3 + [1/(3 + 1/(3 + 1/3))]. So "3 + [something]" where "something" is 33/10. So that would be 3 + 33/10. Wait, no. Wait, 3 is a whole number. So 3 + 33/10. Let's compute that. 3 is 30/10, so 30/10 + 33/10 = 63/10. Wait, but 63/10 is 6 and 3/10. But according to the previous calculation, 3 + 1/(3 + 1/(3 + 1/3)) is 33/10. Wait, but 33/10 is 3 and 3/10. So which one is correct?

Wait, this is a critical point. Let's compute it step by step using fractions:

Start with the innermost fraction: 3 + 1/3 = 10/3.

Then, the next layer: 3 + 1/(10/3). Let's compute 1/(10/3) = 3/10. Then, 3 + 3/10 = 33/10. So now, the expression is 3 + 1/(33/10). Wait, no. Wait, the original expression is 3 + [1/(3 + 1/(3 + 1/3))]. So after computing the innermost part as 10/3, then the next layer is 3 + 1/(10/3) = 33/10. Then, the entire expression is 3 + 1/(33/10). Wait, but that would be 3 + (10/33). Which is 3 + 10/33. Which is 99/33 + 10/33 = 109/33. But that can't be right. Wait, but according to the first method, it's 33/10. 

Wait, this is conflicting. So where is the mistake?

Wait, let's parse the original expression again. The expression is 3 + 1/(3 + 1/(3 + 1/3)). Let's break it down:

1. Compute the innermost fraction: 3 + 1/3 = 10/3. [Step 1]

2. Then, compute the next layer: 3 + 1/(10/3). Since 1/(10/3) = 3/10, so 3 + 3/10 = 33/10. [Step 2]

3. Then, the entire expression is 3 + 1/(33/10). Wait, no. Wait, the original expression is 3 + [1/(33/10)]. Wait, no. Wait, no. The original expression is 3 + 1/[3 + 1/(3 + 1/3)]. So after step 2, the expression is 3 + 1/(33/10). Wait, no. Wait, no. Let me think again.

No, the original expression is 3 + [1 divided by (3 + 1/(3 + 1/3))]. So the denominator is 3 + 1/(3 + 1/3). So after computing 3 + 1/(3 + 1/3) = 3 + 10/3 = 33/10. Then, the entire expression is 3 + 1/(33/10). Wait, but that would be 3 + (10/33). Wait, that doesn't make sense. Because 3 is a whole number added to a fraction. So 3 + something. So 3 is 30/10, and 10/33 is approximately 0.303, so 30/10 is 3. So 3 + 10/33. Which is 109/33. But how did we get here?

Wait, no. Wait, the original expression is 3 + 1/[3 + 1/(3 + 1/3)]. So:

First, compute the innermost fraction: 3 + 1/3 = 10/3.

Then, compute the next layer: 3 + 1/(10/3) = 3 + 3/10 = 33/10.

Then, compute the entire expression: 3 + 1/(33/10). Wait, no. Wait, the entire expression is 3 + 1/[the result of step 2]. Wait, no. Wait, no. Wait, let's parse the original expression again.

Original expression: 3 + 1/(3 + 1/(3 + 1/3))

Let me write this as:

3 + [1 / (3 + 1/(3 + 1/3))]

So inside the denominator, we have 3 + 1/(3 + 1/3). So let's compute the innermost part first: 3 + 1/3 = 10/3. Then, 3 + 1/(10/3) = 3 + 3/10 = 33/10. Then, the denominator is 33/10. Therefore, the entire expression is 3 + 1/(33/10). Wait, but that would be 3 + (10/33). Which is 99/33 + 10/33 = 109/33. But that contradicts the previous steps. So which one is correct?

Wait, this is a problem. So where is the mistake here? Because if we compute the innermost first, then the next layer, and then the entire expression, we get 33/10. But if we compute the entire denominator first, then add 3, we get 109/33. Which is correct?

Wait, let me check with decimal approximations. Let's compute each step numerically.

First step: 3 + 1/3 ≈ 3.333...

Second step: 3 + 1/(3.333...) ≈ 3 + 0.3 ≈ 3.3

Third step: 3 + 1/3.3 ≈ 3 + 0.303 ≈ 3.303...

Wait, but according to the first method, it's 33/10 = 3.3. So 3.303... is 33/10. So that seems correct. So maybe the confusion is arising from how the fractions are combined. 

Wait, the original expression is 3 + 1/(3 + 1/(3 + 1/3)). So according to the first method, it's 33/10. But if I compute it numerically:

3 + 1/(3 + 1/(3 + 1/3)) ≈ 3 + 1/(3 + 1/3.333...) ≈ 3 + 1/(3 + 0.3) ≈ 3 + 1/3.3 ≈ 3 + 0.303 ≈ 3.303...

But 33/10 is 3.3, which is less than 3.303... So there must be a mistake here. Wait, no. Wait, 33/10 is 3.3, but according to the decimal approximation, it's 3.303... So that suggests that my initial steps are wrong. Therefore, I must have made a mistake in the calculation.

Wait, let me check again:

First step: 3 + 1/3 = 10/3 ≈ 3.333...

Second step: 3 + 1/(10/3) = 3 + 3/10 = 33/10 = 3.3

Third step: 3 + 1/(33/10) = 3 + 10/33. Wait, but that would be 3 + 0.303 ≈ 3.303...

But this contradicts the previous steps. Therefore, there's a conflict here. Which one is correct?

Wait, maybe I misapplied the order of operations. The original expression is 3 + 1 divided by (3 + 1 divided by (3 + 1/3)). So according to order of operations, it's 3 + [1 divided by (3 + 1/(3 + 1/3))].

So let's compute the innermost part first: 3 + 1/3 = 10/3.

Then, the next layer: 3 + 1/(10/3) = 3 + 3/10 = 33/10.

Then, the entire expression is 3 + 1/(33/10). Wait, but that would be 3 + 10/33 ≈ 3 + 0.303 ≈ 3.303...

But according to the step-by-step calculation, the innermost part is 10/3, then the next layer is 33/10, so the entire expression is 3 + 1/(33/10). Wait, but that would be 3 + 10/33 ≈ 3.303... But according to the numerical approximation, 3 + 3/(10/3) = 3 + 3/10 = 33/10 = 3.3. Wait, this is confusing.

Wait, let's do it with fractions:

Original expression: 3 + 1/(3 + 1/(3 + 1/3))

Step 1: Compute innermost fraction: 3 + 1/3 = (3*3 + 1)/3 = 10/3.

Step 2: Compute next layer: 3 + 1/(10/3) = 3 + (3/10) = (3*10 + 3)/10 = (30 + 3)/10 = 33/10.

Step 3: Compute entire expression: 3 + 1/(33/10) = 3 + (10/33) = (3*33 + 10)/33 = (99 + 10)/33 = 109/33.

But 109/33 is approximately 3.303..., but according to step-by-step, it's 33/10 ≈ 3.3. Which one is correct?

Wait, this is a problem. There's a discrepancy here. Which one is right?

Wait, let's compute 3 + 1/(3 + 1/(3 + 1/3)) using decimals:

First, compute 3 + 1/3 ≈ 3.333...

Then, compute 3 + 1/(3.333...) ≈ 3 + 0.3 ≈ 3.3

Then, compute 3 + 1/(3.3) ≈ 3 + 0.303 ≈ 3.303...

But according to the first method, it's 33/10 = 3.3. Which one is correct?

Wait, this suggests that there's a mistake in my understanding. Let's check with another approach.

Alternatively, maybe the expression is written without parentheses, implying that the operations are performed from left to right, but that doesn't make sense because fractions have their own order of operations. So, standard order of operations applies: parentheses first, then exponents, multiplication and division left to right, then addition and subtraction left to right.

Given that, the expression is 3 + 1/(3 + 1/(3 + 1/3)). According to order of operations, the innermost parentheses are evaluated first. So 3 + 1/3 = 10/3. Then, moving outward, 3 + 1/(10/3) = 3 + 3/10 = 33/10. Then, the entire expression is 3 + 1/(33/10) = 3 + (10/33). So that's 3 + 10/33 = 109/33. 

But this contradicts the previous result. So where is the problem?

Wait, no. Wait, 3 is a whole number, so 3 + (10/33) is 33/10. Wait, no. Wait, 3 is 30/10, so 30/10 + 33/10 = 63/10. Wait, that's different. Wait, maybe I confused two different operations.

Wait, let's compute step by step:

First, compute the innermost fraction: 3 + 1/3 = 10/3.

Then, compute the next layer: 3 + 1/(10/3) = 3 + 3/10 = (3*10 + 3)/10 = 33/10.

Then, the entire expression is 3 + 1/(33/10). Wait, but that would be 3 + 10/33. Wait, but 3 is a whole number. So 3 + 10/33. But 3 is equal to 33/11. Wait, no. Wait, 3 is 33/11? No, 3 is 33/11? That doesn't make sense. Wait, 3 is 33/11? No, 3 is 33/11? 3 is 33/11? 3 is 33/11? Wait, 3 is 3. 3 is 3. Wait, no.

Wait, 3 is 3. 3 + 10/33. So 3 is 99/33, so 99/33 + 10/33 = 109/33. So that's 33/10 + 10/33? No, that's not how addition works. Wait, no. Wait, 3 is 30/10, and 33/10 is 3.3. Then, 30/10 + 33/10 = 63/10. So 63/10 is 6.3. But according to the decimal approximation earlier, 3 + 3.303... is 6.303..., which is 63/10. So that's correct.

Wait, so where did I go wrong earlier when I thought it was 3 + 1/(33/10) = 3 + 10/33? That was incorrect. Because the entire expression is 3 + [1/(3 + 1/(3 + 1/3))], which is 3 + [1/(33/10)]. Wait, no. Wait, the entire expression is 3 + [1/(33/10)]. Wait, no. Wait, the original expression is 3 + 1/(3 + 1/(3 + 1/3)). So the denominator is 3 + 1/(3 + 1/3) = 33/10. Therefore, the entire expression is 3 + 1/(33/10). Wait, no. Wait, no. Wait, 3 is a whole number. So 3 + 1/(33/10). Wait, no. Wait, 3 is 30/10, so 30/10 + 1/(33/10). Wait, no. Wait, 1/(33/10) is 10/33. So 30/10 + 10/33. Convert to common denominators: 30/10 = 99/33, so 99/33 + 10/33 = 109/33. So that's correct. Therefore, the entire expression is 109/33. 

Wait, so why did I get confused earlier? Because when I thought of 3 + 1/(33/10), I incorrectly combined it as 3 + 10/33, but actually, 3 is 30/10, so 30/10 + 1/(33/10) = 30/10 + 10/33. That's correct.

So, the correct answer is 109/33, which reduces to 109/33. Since 109 is a prime number and 33 is 3*11, there are no common factors. Therefore, m = 109 and n = 33, so m + n = 142.

But this contradicts the previous result of 33/10. So where is the problem here?

Wait, let's do the decimal approximation again:

Original expression: 3 + 1/(3 + 1/(3 + 1/3)).

First, compute innermost: 3 + 1/3 ≈ 3.333...

Next layer: 3 + 1/3.333... ≈ 3 + 0.3 ≈ 3.3

Then, the entire expression: 3 + 1/3.3 ≈ 3 + 0.303 ≈ 3.303...

But according to the fractional calculation, it's 109/33 ≈ 3.303...

But 3.303... is equal to 109/33. Let's compute 109 divided by 33: 33*3 = 99, 109 - 99 = 10, so 109/33 = 3 + 10/33 ≈ 3.303..., which matches the decimal approximation. 

But wait, earlier I thought that 3 + 1/(33/10) would be 3 + 10/33 ≈ 3.303..., but that is incorrect. Because 33/10 is 3.3, so 1 divided by 3.3 is 0.303..., and 3 + 0.303... is 3.303..., which is 109/33. So actually, both methods agree. 

Earlier, I thought that the entire expression was 3 + 1/(33/10), which is 3 + 10/33, but that was wrong because 3 is a whole number, not a fraction. So 3 is 30/10, and 33/10 is 3.3. So 30/10 + 33/10 = 63/10 = 6.3. Wait, but 3.303... is 109/33 ≈ 3.303... So where did I get 63/10?

Wait, no, 3 is 3, not 30/10. Wait, this is a critical point. Let me clarify:

If the expression is 3 + [1/(3 + 1/(3 + 1/3))], then:

First, compute innermost: 3 + 1/3 = 10/3 ≈ 3.333...

Then, next layer: 3 + 1/(10/3) = 3 + 3/10 = 33/10 = 3.3

Then, the entire expression is 3 + 1/(33/10). Wait, no. Wait, the entire expression is 3 + [1 divided by (33/10)]. Wait, no. Wait, no. The original expression is 3 + 1/[33/10]. Which is 3 + 10/33. So 3 is 3, and 10/33 is approximately 0.303... So 3.303... which is 109/33.

But here's the confusion: 3 + 1/(33/10) is different from 3 + 1/(3 + 1/(3 + 1/3)). Because in the original expression, the denominator is 3 + 1/(3 + 1/3), which is 33/10, so 1 divided by that is 10/33. Then 3 is added to that, giving 3 + 10/33 = 109/33.

But if we parse the original expression as 3 + 1/(3 + 1/(3 + 1/3)), then the denominator is 3 + 1/(3 + 1/3) = 33/10, so 1 divided by that is 10/33, and then added to 3, which is 3 + 10/33 = 109/33. So that's correct.

But earlier, when I thought that the expression is 33/10 + 10/33, that's incorrect because 3 is a whole number, not a fraction. So 3 is 3, and 10/33 is 10/33. So adding them is 3 + 10/33, which is 109/33. 

Therefore, the correct answer is 109/33, so m + n = 109 + 33 = 142.

But wait, now I need to confirm this. Let's compute 3 + 1/(3 + 1/(3 + 1/3)) step by step:

1. Innermost: 3 + 1/3 = 10/3.

2. Next layer: 3 + 1/(10/3) = 3 + 3/10 = 33/10.

3. Entire expression: 3 + 1/(33/10) = 3 + 10/33 = (99/33 + 10/33) = 109/33.

Yes, so that's correct. Therefore, the answer is 109/33, so m is 109 and n is 33, so m + n = 142. 

But wait, the problem says "the Irreducible fraction value of [the expression]". So 109/33 is already in its simplest form because 109 is a prime number. Let me check: 109 is a prime number? Yes, 109 is a prime number. 33 is 3*11, which are both primes. So 109 and 33 have no common factors, so the fraction is irreducible. Therefore, m = 109 and n = 33, so m + n = 142.

But why did I initially get confused with 33/10? Because I added 3 and 33/10 as 3 + 33/10 = 63/10. But that was incorrect because 3 is a whole number, not a fraction. Therefore, the correct approach is to convert 3 to 33/10 and then add 10/33. 

Therefore, the correct answer is 109/33, so m + n = 142. 

But to confirm, let's compute 3 + 1/(3 + 1/(3 + 1/3)) numerically:

First, compute 3 + 1/3 ≈ 3.333...

Then, compute 3 + 1/(3.333...) = 3 + 0.3 ≈ 3.3

Then, compute 3 + 1/3.3 ≈ 3 + 0.303 ≈ 3.303...

Which is equal to 109/33 ≈ 3.303..., so that's correct. Therefore, the answer is 109/33, so m + n = 142. 

Therefore, the answer is \boxed{142}.

**Final Answer**
\boxed{142}
</think>

To find the value of the expression \(3 + \frac{1}{3 + \frac{1}{3 + \frac{1}{3}}}\) and express it as a reduced fraction \(\frac{m}{n}\), we start by simplifying the innermost fraction:

1. Compute the innermost fraction: 
   \[
   3 + \frac{1}{3} = \frac{9}{3} + \frac{1}{3} = \frac{10}{3}
   \]

2. Next, we use this result in the next layer of the expression:
   \[
   3 + \frac{1}{\frac{10}{3}} = 3 + \frac{3}{10} = \frac{30}{10} + \frac{3}{10} = \frac{33}{10}
   \]

3. Finally, we use this result in the entire expression:
   \[
   3 + \frac{1}{\frac{33}{10}} = 3 + \frac{10}{33} = \frac{99}{33} + \frac{10}{33} = \frac{109}{33}
   \]

The fraction \(\frac{109}{33}\) is already in its simplest form because 109 is a prime number and 33 is \(3 \times 11\). Therefore, \(m = 109\) and \(n = 33\). The value of \(m + n\) is:
\[
109 + 33 = 142
\]

The final answer is \(\boxed{142}\).
\end{lstlisting}

\end{responsebox}

\end{document}